\documentclass[times,sort&compress,3p]{elsarticle}
\journal{Journal of Multivariate Analysis}
\usepackage[labelfont=bf]{caption}

\usepackage{amsmath,amsfonts,amssymb,amsthm,booktabs,color,epsfig,graphicx,hyperref,url}
\renewenvironment{proof}{{\textbf{Proof.} }}{\hfill $\Box$ \\}

\theoremstyle{plain}
\newtheorem{theorem}{Theorem}

\newtheorem{lemma}{Lemma}
\newtheorem{corollary}{Corollary}

\theoremstyle{definition}
\newtheorem{definition}{Definition}
\newtheorem{remark}{Remark}

\usepackage[utf8]{inputenc}
\usepackage{graphicx}
\usepackage{amsmath}
\usepackage{amsfonts}
\usepackage{amsthm}
\usepackage{amssymb}
\usepackage[]{todonotes}
\usepackage{float}
\usepackage{relsize}
\usepackage{bbm}
\usepackage{subcaption}
\usepackage{algpseudocode}
\usepackage{algorithm} 
\usepackage{bbold}
\usepackage{tikz}
\usepackage{circuitikz}

\newtheorem*{conjecture}{Conjecture}
\newtheorem{problem}{Problem}

\newcommand{\conjectureRef}{\hyperlink{hopkinsConjecture}{Low Degree Conjecture} }

\renewcommand\appendixautorefname[1]{}

\newcommand{\X}{\mathbf{X}}

\newcommand{\bA}{\mathbf{A}}

\newcommand{\bE}{\mathbf{E}}
\newcommand{\bI}{I}
\newcommand{\bX}{\mathbf{X}}
\newcommand{\bY}{\mathbf{Y}}

\newcommand{\bu}{\mathbf{u}}
\newcommand{\bv}{\mathbf{v}}
\newcommand{\bg}{\mathbf{g}}

\newcommand{\be}{\mathbf{e}}
\newcommand{\bx}{\mathbf{x}}

\newcommand{\bnu}{\mathbf{\nu}}

\newcommand{\cQ}{\mathcal{H}_0}
\newcommand{\cP}{\mathcal{H}_1}
\newcommand{\cU}{\mathcal{U}}
\newcommand{\cO}{\mathcal{O}}
\newcommand{\cN}{\mathcal{N}}
\newcommand{\cB}{\mathcal{B}}
\newcommand{\BR}{\mathcal{BR}}
\newcommand{\cD}{\mathcal{D}}
\newcommand{\cF}{\mathcal{F}}

\newcommand{\bbE}{\mathbb{E}}
\newcommand{\bbP}{\mathbb{P}}

\newcommand{\bbI}{I}

\newcommand{\Natural}{\mathbb{N}}
\newcommand{\Real}{\mathbb{R}}
\newcommand{\bbS}{\mathbb{S}}

\newcommand{\ainner}{\langle \bu, \bu^* \rangle}
\newcommand{\ainnerdash}{\langle \bu', \bu^* \rangle}
\newcommand{\ainnerhat}{\langle \hat{\bu}, \bu^* \rangle}

\newcommand{\Data}{\mathcal{D}}
\newcommand{\Normal}{\mathcal{N}}
\newcommand{\expectation}{\mathbb{E}}
\newcommand{\probability}{\mathbb{P}}
\newcommand{\projectustar}[1]{\left\langle #1, \bu^* \right\rangle}

\newcommand{\sigmainner}{\sqrt{1 - \ainner^2}}
\newcommand{\inner}[1]{\left\langle #1 \right\rangle}
\newcommand{\average}[1]{\sumin \frac{#1}{n}}

\DeclareMathOperator*{\argmax}{arg\,max}

\newcommand{\correction}{(\bI_d - \bu^* {\bu^*}^\top - \be_2 \be_2^\top)}
\newcommand{\correctiontwo}{(\bI_d - \bu^* {\bu^*}^\top)}
\newcommand{\mynorm}[1]{\left\|#1\right\|_2}
\newcommand{\myonenorm}[1]{\left\|#1\right\|_1}
\newcommand{\opnorm}[1]{\left\|#1\right\|_{op}}
\newcommand{\myabs}[1]{\left|#1\right|}

\newcommand{\oneOrclizNorm}[1]{\left\|#1\right\|_{\psi_1}}
\newcommand{\twoOrclizNorm}[1]{\left\|#1\right\|_{\psi_2}}
\newcommand{\epsilonNet}{\mathcal{N}_{\epsilon, d}}
\newcommand{\sumin}{\sum_{i = 1}^n}

\usepackage{dsfont}
\newcommand{\indicator}{\mathds{1}}

\newcommand{\bigO}[1]{\mathcal{O}\left( #1 \right)}
\newcommand{\smallO}[1]{o\left( #1 \right)}
\newcommand{\bigOs}[1]{\widetilde{\mathcal{O}}\left( #1 \right)}

\newcommand{\bigTheta}[1]{\Theta\left( #1 \right)}
\newcommand{\bigThetas}[1]{\widetilde{\Theta}\left( #1 \right)}

\newcommand{\bigOmega}[1]{\Omega\left( #1 \right)}
\newcommand{\bigOmegas}[1]{\widetilde{\Omega}\left( #1 \right)}

\newcommand{\relutwo}[1]{\max\left\{0, #1\right\}^2}

\newcommand{\ldlr}{\mynorm{L_{d}^{\leq D}}}

\newcommand{\ubar}{\Bar{\bu}}
\newcommand{\uhat}{\hat{\bu}}
\newcommand{\ustar}{\bu^*}
 \begin{document}

\begin{frontmatter}

\title{
Recovering Imbalanced Clusters via Gradient-Based Projection Pursuit
}

\author[1]{Martin Eppert \corref{mycorrespondingauthor}}
\author[2]{Satyaki Mukherjee}
\author[1]{Debarghya Ghoshdastidar}

\address[1]{
Technical University of Munich
School of Computation, Information and Technology - I7
Boltzmannstr. 3 
85748 Garching b. München 
Germany}
\address[2]{
National University of Singapore
Level 4, Block S17
10 Lower Kent Ridge Road
Singapore 119076 
}
\cortext[mycorrespondingauthor]{Corresponding author. Email address:\url{martin.eppert@tum.de}}
\begin{abstract}
    Projection Pursuit is a classic exploratory technique for finding interesting projections of a dataset.
    We propose a method for recovering projections containing either Imbalanced Clusters or a Bernoulli-Rademacher distribution using a gradient-based technique to optimize the projection index.
    As sample complexity is a major limiting factor in Projection Pursuit, we analyze our algorithm's sample complexity within a Planted Vector setting where we can observe that Imbalanced Clusters can be recovered more easily than balanced ones.
    Additionally, we give a generalized result that works for a variety of data distributions and projection indices.
    We compare these results to computational lower bounds in the Low-Degree-Polynomial Framework.
    Finally, we experimentally evaluate our method's applicability to real-world data using FashionMNIST and the Human Activity Recognition Dataset, where our algorithm outperforms others when only a few samples are available.
\end{abstract}
\end{frontmatter}
\section{Introduction}
Projection Pursuit was introduced in \citet{PPFriedmanTukey} as a method for finding maximally interesting projections, as determined by the histogram of the projected data. 
A function known as the projection index is used to assess the interestingness of the data projected onto a subspace.
A projection is then obtained that optimizes the projection index to reveal the structure of the data.
The methodology of projection pursuit has proven itself useful as a sub-procedure in many statistical analyses.
One area of application for projection pursuit is in clustering algorithms, where finding projections that minimize kurtosis or maximize the $\ell_1$-norm can reveal cluster structures in many settings~\cite{Pena2000, davis2021clustering, Pena2001a}.
Projection pursuit is also used for outlier detection~\cite{Pena2001b, Pena2007, nicolaKurtosis}.
Here, finding projections that maximize kurtosis can reveal small clusters that may correspond to outliers.
In other applications, projection pursuit is used to discover non-Gaussian independent components from a linear mixture~\cite{fastICA, fastICAgb}.
This is done by finding non-Gaussian projections of the data, which can be done by maximizing excess kurtosis or using other suitable measures of non-Gaussianity.
Another related application of projection pursuit is the recovery of sparsely used dictionaries.
Here, it is generally assumed that an orthonormal basis exists, and the data will be sparse if represented in this basis.
Finding projections that minimize $\ell_1$-norm~\cite{bai2019subgradient} or maximize kurtosis~\cite{DictionaryLearningKurtosis} can be used to reveal individual vectors in the basis.

There exists a plethora of approaches to optimizing projection indices.
Iterative methods such as gradient descent are commonly used in projection pursuit~\cite{bai2019subgradient, DictionaryLearningKurtosis, QuSW14}.
Originating from Independent Component Analysis, the iterative FastICA algorithm is commonly used to optimize projection indices~\cite{fastICA,fastICAgb}.
The initialization of iterative methods is also essential to optimize their performance.
Methods for initializing projection pursuit algorithms include random directions~\cite{ica_random_init}, normalized samples form the data~\cite{spielman2012exact}, eigenvalues of kurtosis matrices~\cite{KurtosisLoperfido18, Pena2010} and skewness matrices~\cite{nicolaSkewness}.
\citet{Arevalillo2021} also discusses these initializations and their sampling properties.
A different approach to optimizing projection indices, compared to iterative methods, is using properties of higher-order moment matrices to optimize projection indices.
This is typically done by utilizing the eigenvector corresponding to the largest or smallest eigenvalue of the moment matrix. 
There are many different formulations using eigen-/singular-values of third~\cite{nicolaSkewness, LOPERFIDO2015202, Kollo2008} and fourth~\cite{Pena2010, nicolaKurtosis, mao2022optimal} order moment matrices.
Computational complexity is a key consideration from an implementation perspective.
Frequently, Gradient-based methods are favored due to their low computational cost.
In contrast, approaches that rely on the eigenvalues or singular values of moment matrices require the construction of large matrices.
The use of methods like quasi-Newton needs the construction of the Hessian~\cite{Friedman01031987}.
The materialization of these matrices becomes prohibitively expensive in large-dimensional settings.
Another factor increasing the computational complexity of iterative methods is the number of iterations to convergence~\cite{QuSW14}.

Although finding an optimum of the projection index is in many cases theoretically possible, even with a small number of samples (e.g., via brute-force enumeration of projections), it is generally computationally intractable~\cite{mao2022optimal, davis2021clustering}.
This is commonly called a statistical-to-computational gap when a statistical problem is solvable, but all algorithms are computationally intractable.
Assessing the optimal sample complexity that efficient algorithms can achieve in a specific setting is crucial in guiding algorithm development.
Although there is currently no way to show lower bounds for all efficient algorithms, there exist lower bounds for specific classes, such as Low-Degree Polynomials, of algorithms that are conjectured optimal for specific classes of problems~\cite{hopkins2018statistical, kunisky2019notes}.

The planted vector setting~\cite{mao2022optimal, dudeja2024statisticalcomputational, hopkins2016fast} is a well-suited setting for studying the sampling properties of projection pursuit.
The planted vector problem is a specific instance within the broader class of planted problems concerned with recovering a hidden (planted) signal embedded in random noise~\cite{alon1998, hopkins2016fast}.
In this setting, the objective is to identify the planted vector from noisy observations, leveraging the structural properties of the signal and the noise.
In the planted vector setting, projecting the data onto a specific direction —referred to as the signal direction— reveals the hidden signal, while projections onto orthogonal directions yield data that follows a Gaussian distribution.
This setting is widely studied because it models a range of classical statistical problems, such as clustering in Gaussian mixture models, which can be framed as instances of planted vector problems.
Statistical procedures often start by whitening the data to eliminate spurious correlations, a step naturally captured in the planted vector framework, where the data is typically assumed to have unit covariance.

Due to its restrictive nature, the planted vector setting lends itself well to studying lower bounds on the sample complexity of algorithms~\cite{davis2021clustering, QuSW14, mao2022optimal, hopkins2016fast}.
\citet{mao2022optimal} show that recovering a $p$-sparse planted vector in a $d$-dimensional space requires at least $\bigOmegas{d^2 p^2}$ samples for Low-Degree Polynomials to succeed.
\citet{dudeja2024statisticalcomputational} considers the more general case of recovering a planted vector where the first $(k-1)$-moments match a Gaussian distribution.
In this setting, at least $\bigOmegas{d^{k/2}}$ samples are required for Low-Degree Polynomials to succeed.
\citet{diakonikolas2024sumofsquareslowerboundsnongaussian} extend these results to the Sum-Of-Squares framework.
\subsection{Our Contributions}
Although gradient-based algorithms are commonly used in practice, their sampling properties are still difficult to analyze compared to the analysis of spectral methods.
For example, \citet{QuSW14} note in their paper on recovery of planted vectors with gradient-based methods that the sample complexity of their method is significantly lower than their presented bounds.
This motivates us to design a gradient-based projection pursuit algorithm that is easy to analyze in the planted vector setting.
We provide lower bounds on the required sample complexity for this algorithm, which closely match our simulation results and are also competitive with known results in the planted vector setting.
The computational complexity of projection pursuit algorithms is also a major roadblock to their widespread adoption.
The computational complexity of our algorithm scales linearly with $\bigO{n d}$, which is significantly lower than other methods.
This is also significant since we do not need to run more than one gradient step on each sample, avoiding the computational overhead of running many iterations.

We further demonstrate the experimental efficacy of this algorithm in two different settings:
First, we demonstrate the efficacy of the presented algorithm in recovering a $d$-dimensional planted vector containing two Imbalanced Clusters with $p$ denoting the cluster probability.
In this setting, use the projection index ${\phi(x) = \max\{0,x\}^2}$ (ReLU2) for which we prove that $\bigOs{d^2 p^2}$ samples are sufficient for gradient ascent to recover the signal direction.
As cluster imbalance increases, the smaller cluster moves further from the origin while the larger cluster approaches the origin.
We also apply the algorithm to recover a Bernoulli-Rademacher planted vector using kurtosis as a projection index for which $n = \bigOmegas{d^3 p^4}$ are sufficient.

We employ multiple techniques to improve our algorithm's sample complexity and allow a tight analysis of its sample complexity.
First, we propose using fresh mini-batches in each iteration of gradient ascent~\cite{bertsekas1996incremental}.
Suppose we reuse the same dataset for each step of gradient ascent.
In that case, more samples are necessary to ensure that the optimization problem is smooth enough so that gradient ascent does not get stuck in local maxima.
Using mini-batches mitigates this problem but comes at a significant cost, as we require new samples in each iteration.
However, we prove that only a few steps are necessary to converge, and thus, the impact of resampling on the sample complexity is small.

Additionally, we suggest initializing gradient ascent with normalized samples from the dataset, similar to a technique used in \citet{spielman2012exact}. 
Many planted vectors, such as Imbalanced Clusters, have longer tails than the standard Gaussian distribution.
Exploiting this, we demonstrate that with high probability, it is possible to find an initialization closer to the signal direction than what could be found using a random initialization.
As estimating the gradient for a direction closer to the signal directions generally becomes easier, we require fewer samples to estimate the gradient accurately.

To further study the setting of planted Imbalanced Clusters, we study lower bounds on the sample complexity of recovering the planted vector using any polynomial time algorithm.
Our study of the setting where the planted vector contains two clusters, with one being significantly larger than the other, differs from commonly studied planted vector settings in which the planted vector is symmetric, such as the Bernoulli-Rademacher in \citet{mao2022optimal} and \citet{QuSW14}.
We prove a computational lower bound close to the sample complexity required for the gradient-based algorithm.
For this lower bound in the setting of Imbalanced Clusters, we extend \citet{mao2022optimal}, which uses the framework of low-degree-polynomials~\cite{kunisky2019notes, hopkins2018statistical}.
Here, we obtain a lower bound on the sample complexity of $n = \bigOmegas{d^{1.5} p}$.
\citet{dudeja2024statisticalcomputational} prove a similar result, which is not applicable in the setting where $p$ decreases with respect to the problem's dimensionality. 
In contrast, our result captures the setting where $p$ decreases proportionally to the dimension.
While the number of samples required by our method is not optimal, it is sufficiently close.

Finally, to motivate the pursuit of imbalanced projections, we show that by applying our algorithm to the FashionMNIST dataset~\cite{xiao2017fashionmnist} and the Human Activity Recognition dataset~\cite{human_activity_recognition_using_smartphones_240}, we obtain directions that reveal clusters in the data that correspond to their labels.
The utility of these projections is measured by how much information they provide on the class labels.
Especially with the projection index $\phi(x) = \max\{0,x\}^2$, we observe that even with a small number of samples, it is possible to find projections that reveal the class structure of the dataset by separating one class from the others.
\subsection{Document Structure}
In \autoref{setup}, we define the planted vector setting.
\autoref{gbAlgo} describes the proposed algorithm.
\autoref{scBounds} describes a general analysis of the gradient-based algorithm.
\autoref{Relu2Section} and \autoref{KurtosisSection} demonstrate results on imbalanced clusters and Bernoulli-Rademacher planted vectors.
\autoref{ldplb} discusses low-degree lower bounds for imbalanced clusters planted vectors.
\autoref{experiments} contains simulations and experiments on real data. 
Further experiments on synthetic data can be found in \autoref{supplement}.
\section{Main Results}
We use standard asymptotic notation $o(\cdot)$, $\mathcal{O}(\cdot)$, $\Theta(\cdot)$, $\Omega(\cdot)$ and $\widetilde{\mathcal{O}}(\cdot)$, $\widetilde{\Omega}(\cdot)$ which hides logarithmic factors.
$\bbS_{d-1}$ denotes the $d$-dimensional unit sphere.
$\cU(\cdot)$ denotes the uniform distribution and $\cN(\mu, \sigma^2)$ denotes the standard normal distribution with mean $\mu$ and standard deviation $\sigma$.
\subsection{Setup}
\label{setup}
We follow the literature on recovery of planted vectors \cite{mao2022optimal, dudeja2024statisticalcomputational, hopkins2016fast}, which gives a simplified formulation of the data assumption in Projection Pursuit.
Throughout the paper, we use $n$ to indicate the number of samples present and $d$ to indicate the dimensionality of the data.
In the Planted Vector Setting, the model is constructed as follows:
\begin{definition}[Planted Vector Setting]
    \label{planted_vector_setting}
    We say $\bx\sim\cD_\cF$ if, 
    $\bx \sim \mathcal{N}(\bnu \, \ustar, \bI_d - \ustar{\ustar}^\top)$.
    Given the random variable $\nu \sim \cF$ for some distribution $\cF$
    and a fixed but unknown direction $\ustar$.
\end{definition}
The distribution of $\nu$ is generally defined to follow a non-Gaussian distribution with unit variance.
Later, we will consider the setting where $\nu$ follows a distribution containing either two Imbalanced Clusters or a sparse distribution.
If $d$ is large, then if the data is projected in a random direction, it will be approximately Gaussian, but if projected in the direction $\bu^*$, the structure of $\nu$ can be observed.
\subsection{Gradient-Based Algorithm}
\label{gbAlgo}
\vspace{-0.1cm}
This section describes \autoref{twoStepAlgorithm}, a gradient-based algorithm for optimizing differentiable projection indices.
Here, we assume that we are given access to a dataset containing a planted vector as in \autoref{planted_vector_setting}.
The recovery of the signal direction ($\bu^*$) is done by finding an (approximate) solution to the following optimization problem where $\psi$ is the projection index.
\vspace{-0.2cm}
$$\hat{\bu} = \max_{\bu \in \bbS_{d-1}} \sumin \frac{\psi(\inner{\bX_i, \bu})}{n}$$
\vspace{-0.2cm}\\
This will be done by performing gradient ascent using a different projection index $\phi$ using multiple initializations.
Then $\psi$ is used to pick the best direction $\hat{\bu}$.

Two key ideas are used in the algorithm design.
The first idea is to use multiple initializations from the dataset by using $\bu_i = \frac{\bX_i}{\mynorm{\bX}}$ as initialization inspired by \citet{spielman2012exact, QuSW14}
Intuitively, initializing closer to the planted vector allows for a more accurate estimation of the gradient, which allows a decrease in sample complexity.
If the planted vector's distribution has heavier tails than the normal distribution, using normalized samples from the distribution can provide initializations closer to the signal direction than uniformly random samples.
For example if $\bbP\left[\nu = \sqrt{1/p}\right] = p$ then with probability $p$ we have an initialization for which $\ainner \approx \sqrt{1/ (d p)}$.
Instead of choosing $\bu$ uniformly at random, we only have $\ainner \approx \sqrt{1/d}$, which can be much worse in the case when $p$ is sufficiently small.
This method requires using many initializations simultaneously without knowing which initializations are close to the signal direction.
Successfully converged projections are then detected and returned.
This initialization scheme is shown in \autoref{twoStepAlgorithm}.

The second idea is to use minibatches to avoid being stuck in local optima.
As previously noted, this comes at the cost of needing new samples for each gradient ascent step.
Thus, speeding up convergence is necessary to decrease the algorithm's sample complexity.
We do this by using the Riemannian gradient.
$
    \left( \bI_d - \bu \bu^{\top}\right)
    \frac{\partial }{\partial \bu}
    \left(
        \sum_{i = 1}^{n} \frac{\phi\left(\inner{\bX_i, \bu}\right)}{n}
    \right)
$
instead of the gradient itself, which allows us to decrease the number of steps needed to converge.

Using the Riemannian gradient itself also has a downside.
If $\ainner$ becomes sufficiently large, we cannot guarantee that a gradient step does not degrade the current estimate of the signal direction.
Thus, we use a schedule for the learning rate $\eta$, decreasing $\eta$ once we are close to convergence.
The algorithm is presented in \autoref{twoStepAlgorithm}, which calls a subroutine described in \autoref{the_alg}.
\autoref{the_alg} runs Riemannian gradient ascent with a large learning rate $\eta_1$ and then extracts the solution with the largest value for the projection index, ensuring that a good solution is found.
A second run of \autoref{the_alg} with a lower learning rate $\eta_2$ is used to fine-tune the projection to find a close estimate of the signal direction $\bu^*$.
As we use multiple initializations, we end up with multiple estimates of the signal direction. 
Additionally, we are unsure if an estimate may have diverged during gradient ascent.
Thus, we use the second projection index $\psi(\cdot)$ to pick the best estimate of the signal direction.
\vspace{-0.5cm}
\\
\noindent
\begin{minipage}[t]{.51\textwidth}
\begin{algorithm}[H]
    \caption{Two-Step Gradient Ascent Algorithm}
    \label{twoStepAlgorithm}
    \begin{algorithmic}[1]\Function{two\_step\_gradient\_ascent}{$\bX, n, n_{init}, s, \eta_1, \eta_2$} 
        \For{$j = 1 ... n_{init}$}
            \State $\bu_{j} \gets \frac{ \bX_{j} }{ \mynorm{\bX_{j}} }$
        \EndFor
        \State $\hat{\bu} \gets 
        \Call{gradient\_ascent}{
            \{\bX_i\}_{i = n_{init} + 1}^{n_{init} + n s}, 
            \bu, n, 
            \eta_1, 
            s
        }$
        \vspace{0.1cm}
        \State $\hat{\bu} \gets 
        \Call{gradient\_ascent}{
            \{\bX_i\}_{i = n_{init} + n s + 1}^{n_{init} + 2 n s}, 
            \hat{\bu}, n, 
            \eta_2, 
            s
        }$
        \vspace{0.1cm}
        \State \Return $\argmax_{\hat{\bu} \in \{\hat{\bu}_{j} | j \in [n_{init}]\}} \sum_{k = 1}^{n} \frac{\psi(\inner{\bX_k, \hat{\bu}})}{n}$
    \EndFunction
    \end{algorithmic}
\end{algorithm}
\end{minipage}
\begin{minipage}[t]{.49\textwidth}
\begin{algorithm}[H]
    \caption{Gradient Ascent}
    \label{the_alg}
    \begin{algorithmic}[1]
        \Function{gradient\_ascent}{$\bX, \bu, n, \eta, s$} 
            \label{subroutine}
            \For{$i = 0 ... (s-1)$}
                \State Choose $\bar{\bX} \gets \{\bX_k\}_{k = n i}^{n (i + 1)}$
                \For{$j = 1 ... n_{init}$}
                    \State Calculate 
                    \State \hspace{0.25cm} $\bg \gets
                    \left( \bI_d - \bu_{i, j} \bu_{i, j}^{\top}\right)
                    \frac{\partial }{\partial \bu_{i, j}}
                    \left(
                        \sum_{k = 1}^{n} \frac{\phi\left(\inner{\bar{\bX}_k, \bu_{i, j}}\right)}{n}
                    \right)
                    $
                    \State Update $\ubar_{i, j} \gets \bu_{i, j} + \eta \bg$
                    \State Renormalize $\bu_{i+1, j}\gets \frac{ \ubar_{i+1, j} }{ \|\ubar_{i+1, j}\|_2 }$
                \EndFor
            \EndFor
            \For{$j = 1 ... n_{init}$}
                \State $\hat{i} \gets 
                    \argmax_{i \in [s]}
                    \sum_{k = 1}^{n} \frac{\psi(\inner{\bX_k, \bu_{i, j}})}{n}
                $
                \State $\hat{\bu}_j \gets \bu_{\hat{i}, j}$
            \EndFor
            \State \Return $\hat{\bu}$
        \EndFunction
    \end{algorithmic}
\end{algorithm}
\end{minipage} 
\subsection{Sample Complexity Bounds}
\label{scBounds}
Next, we will highlight a method of studying the sample complexity of the Gradient Ascent Subroutine (\autoref{the_alg}) when specified towards a planted vector distribution and a projection index.
We state three assumptions that must be fulfilled by the setting and the projection index to demonstrate convergence.
This analysis can then be applied to both uses of \autoref{the_alg} to complete the analysis of \autoref{twoStepAlgorithm}.
Later, in \autoref{Relu2Section} and \autoref{KurtosisSection}, we show that \autoref{twoStepAlgorithm} can recover planted vectors with close to optimal sample complexity.

\autoref{mainTheorem} gives a convergence result for arbitrary $\phi(\cdot), \psi(\cdot)$ and a planted vector distribution.
To apply \autoref{mainTheorem}, we have to demonstrate the following preconditions hold.

\autoref{assInitialization} of \autoref{mainTheorem} guarantees that at least one initialization is close enough to the signal direction, providing a good starting point for our algorithm.
\autoref{assGradient} ensures that the gradient estimates are sufficiently accurate for each step.
This also ensures that renormalization does not decrease $\ainner$.
Finally, \autoref{assTestability} ensures that the projection index $\psi(\cdot)$ can be used to (sample-)efficiently test if an initialization has converged.
In most cases, choosing $\psi = \phi$ is entirely sufficient, but selecting a convenient $\psi$ can oftentimes drastically ease the analysis.
This is necessary in the last step of the algorithm to select a converged estimate.
In the following we will use $g_{\bu}(\bx) := (\bI - \bu \bu^\top) \frac{\partial \phi(\inner{\bx, \bu})}{\partial \bu}$ for simplicity.
\begin{lemma}
\label{mainTheorem}
Let $\bX \sim \cD^n$.
For $s = \bigOmega{\log(d)}$ steps, $\delta \geq 0$, $1 > b > b-\delta > a > 0$, $\delta > 0$ and $n > 0$, if
\begin{enumerate}
    \item \label{assInitialization}
    Given $\bu_{0, i}$ for $i \in [n_{init}]$ 
    then with probability at least $1 - \smallO{1}$
    $$
        \max_i \inner{\bu_i, \bu^*} \geq a$$
    
    \item \label{assGradient}
    For an arbitrary constant $c_0 > 0$, if $\ainner \in (a, b)$ then
    $$
    \frac{
        \ainner + \eta \inner{\frac{\sumin g_\bu(\bX_i)}{n}, \bu^*}
    }{
        \sqrt{1 + \eta^2 \frac{\sumin g_\bu(\bX_i)}{n}}
    }
    \geq
    (1 + c_0)
    \ainner
    $$
    with probability at least $1 - \bigO{\frac{1}{s}}$.

    \item \label{assTestability}
    There exists a threshold $t$ where for all $\bu \in \bbS_{d-1}$
    if $\inner{\bu^*, \bu} \geq b$ then $\sumin \frac{\psi(\inner{\bX_i, \bu})}{n} \geq t$ 
    and
    if $\ainner \leq b - \delta$ then $\sumin \frac{\psi(\inner{\bX_i, \bu})}{n} \leq t$, with probability at least $1 - \smallO{1}$.
\end{enumerate}

Then, \autoref{the_alg} returns $\hat{\bu}$ such that $\max_{i \in [n_{init}]} \inner{\hat{\bu}_i, \bu^*} \geq b - \delta$ with a total of $\widetilde{\mathcal{O}}(n)$ samples steps with probability at least $1 - \smallO{1}$.
\end{lemma}

The proof of \autoref{mainTheorem} can be found in \autoref{proofMainTheorem}.
Thus, to analyze the performance of \autoref{twoStepAlgorithm}, we apply \autoref{mainTheorem} once for each execution of \autoref{the_alg}. \subsection{Application of \autoref{mainTheorem} to Imbalanced Clusters}
\label{Relu2Section}
Here, we will focus on a $\cB(p)$ containing two imbalanced clusters with an imbalance parameter $p$.

\begin{definition}[Imbalanced Clusters]
    \label{icDefinition}
    We say $\nu \sim \cB(p)$, with $p \in (0,1)$, if 
    $
        \nu = 
        \begin{cases}
          \sqrt{(1-p) / p}, & \text{with probability } p, \\
          - \sqrt{p / (1-p)}, & \text{with probability } (1-p).
        \end{cases}
    $
\end{definition}
The cluster centers are chosen so that the mean is zero and the variance is one.
Due to the data having unit variance, methods such as PCA cannot recover the planted vector and need Projection Pursuit methods.
Note that for smaller $p$, the first cluster moves further away from the origin, and the second cluster shifts closer to the origin.
This behaves similarly to the Bernoulli-Rademacher distribution.
Here, we will be interested in the parameter $p \in (\frac{1}{\sqrt{d}}, \frac{1}{2})$.
For larger $p > \frac{1}{2}$, the same results follow by symmetry with the notable exception of $p = \frac{1}{2}$ where the clusters are perfectly balanced.
We choose to use the projection index $\phi(x) = \max\{0, x\}^2$.
In \autoref{relu2_sc} we demonstrate bounds on the sample complexity of \autoref{twoStepAlgorithm} using $\phi(x) = \psi(x) = \max\{0, x\}^2$.

\begin{theorem}
    \label{relu2_sc}
    For arbitrary $\beta > 0$ and $p \in \left( \frac{1}{2}, \frac{1}{\sqrt{d}}\right)$ there exist 
$\eta_1 = \bigOmega{\sqrt{d} p}$, $\eta_2 = \bigTheta{1}$, $s = \Theta(\log(d))$, $n_{init} = \Omega\left( 1 / p \right)$ and $n = \bigThetas{d^2 p^2}$ such that for sufficiently large $d$ and sufficiently small $p$, Projection Pursuit using \autoref{twoStepAlgorithm} 
    with $\bX \sim \cD_{\cB(p)}^n$ and a Projection Index $\phi(x) = \max\{0, x\}^2$
    will output $\hat{\bu}$ such that $\inner{\hat{\bu}, \bu^*} \geq 1 - \beta$ 
    with probability at least $1 - \smallO{1}$ utilizing a total of $\bigThetas{d^2 p^2}$ samples.
\end{theorem}
The proof of \autoref{relu2_sc} can be found in \autoref{relu2_sc_proof}.
\subsection{Application of \autoref{mainTheorem} to Bernoulli-Rademacher Planted Vectors}
\label{KurtosisSection}
Other commonly studied settings are the Bernoulli-Rademacher and Bernoulli-Gaussian settings~\cite{mao2022optimal, hopkins2016fast, DictionaryLearningKurtosis}.
These are both sparse distributions, i.e., are $0$ with probability $1-p$, thus are of particular interest in compressed sensing~\cite{spielman2012exact, DictionaryLearningKurtosis}.
We will prove that a Bernoulli-Rademacher planted vector can be recovered using gradient-based techniques.

\begin{definition}[Bernoulli-Rademacher \cite{spielman2012exact}]
    \label{brDefinition}
    We say $\nu \sim \BR(p)$, with $p \in (0,1)$, if
    \begin{equation*}
        \nu = 
        \begin{cases}
          \sqrt{1 / p},&\text{with probability } p / 2, \\
          - \sqrt{1 / p},&\text{with probability } p / 2, \\
          0,&\text{with probability }(1\hspace{-0.1cm}-\hspace{-0.1cm}p).
        \end{cases}
    \end{equation*}
\end{definition}

We demonstrate, that \autoref{twoStepAlgorithm} using the projection index $\phi(x) = x^4$ can recover the planted vector using $n = \widetilde{\mathcal{O}}(d^3 p^4)$ samples.

\begin{theorem}
    \label{kurtosis_sc}
    For arbitrary $\beta > 0$ there exist 
$\eta_1 = \bigOmega{d p^2}, \eta_2 = \bigTheta{1}, s = \bigOmega{\log(d)}, n_{init} = \Theta\left( 1 / p \right)$ and ${n = \bigThetas{d^3 p^4}}$ such that for sufficiently large $d > 0$ and sufficiently small $\frac{1}{3} > p > 0$, Projection Pursuit using \autoref{twoStepAlgorithm} 
    with $\bX \sim \cD_{\BR(p)}^n$, $\phi(x) = x^4$ and $\psi(x) = -|x|$
    will output $\hat{\bu}$ such that $\inner{\hat{\bu}, \bu^*} \geq 1 - \beta$ 
    with probability at least $1 - \smallO{1}$ utilizing a total of $\bigThetas{d^3 p^4}$ samples.
\end{theorem}
The proof can be found in \autoref{kurtosisProofs}. 
\section{Statistical Computational Lower Bounds of the Planted Vector Setting}
\label{ldplb}
Here, we study whether gradient-based methods are optimal in the sense of matching computational lower bounds.
For this, we compare the sample complexity of gradient-based methods to computational lower bounds, which assess the minimum sample complexity required for any computationally efficient algorithm (i.e., computable in polynomial time) to recover the planted vector, as defined in \autoref{planted_vector_setting}.
As this is not tractable for such a general class of algorithm, there have been rigorous results in more limited settings such as lower bounds for the statistical query model~\cite{kearnsStatisticalQuery1993}, sum of squares hierarchies~\cite{sumofsqares} and the Low Degree Polynomial Framework~\cite{kunisky2019notes, hopkins2018statistical}.
Here, we will focus on the framework of Low Degree Polynomials.

Generally, the Low-Degree Polynomial Framework uses Low-Degree Polynomials as a surrogate for efficiently computable algorithms to determine whether an efficient algorithm exists for deciding a hypothesis testing problem.
We will utilize this to obtain lower bounds on the sample complexity of efficiently computable tests.
Bounds can be obtained using the optimality of the likelihood ratio test~\cite{kunisky2019notes}.
Let $L := \frac{d\cP}{d\cQ}$ be the likelihood ratio.
\citet{Neyman1992} shows that thresholding the likelihood ratio $L$ is an optimal test, thus allowing reasoning about computational lower bounds.
The Low Degree Polynomial Framework focuses on the degree-$D$ likelihood ratio $L_d^{\leq D}$ which is defined as the likelihood ratio projected onto the subspace of polynomials of degree at most $D$, where $D$ is low, i.e., logarithmic in the size of the problem.
By demonstrating that a likelihood ratio test using $L_d^{\leq D}$ fails, we can prove that no polynomial of degree $\leq D$ can be used to construct a test to distinguish $\cP$ and $\cQ$.
This is summarized in the following conjecture.
\begin{conjecture}[Low Degree Conjecture \cite{hopkins2018statistical}]
    \hypertarget{hopkinsConjecture}
    For "sufficiently nice" sequences of probability measures $\cQ$ and $\cP$, if there exists $\epsilon > 0$ and degree $D \geq \log(d)^{1+\epsilon}$ for which $\|L_d^{\leq D}\|$ remains bounded as $d \to \infty$, then there is no polynomial-time algorithm $f$ for which if $\bX \sim \cQ$ then $f(\bX) = \cQ$ and if $\bX \sim \cP$ then $f(\bX) = \cP$ with high probability. 
\end{conjecture}
\subsection{Planted Vectors with Imbalanced Clusters}
To our knowledge, computational lower bounds have not been studied for the planted vector in the Imbalanced Clusters setting.
To obtain lower bounds on the sample complexity needed to recover a close estimate of the planted vector in polynomial time, we follow a three-step procedure following the method used in \citet{mao2022optimal}.
First, we formulate a hypothesis testing problem in \autoref{testing}, which tests between a Gaussian distribution and the Planted Vector distribution as defined in \autoref{planted_vector_setting}.
Then, we demonstrate computational lower bounds on \autoref{testing} in the Low Degree Polynomial Framework. 
Finally, we extend the computational lower bounds to the estimation problem of finding a direction $\hat{\bu}$ close to the signal direction $\bu^*$ such that $\ainnerhat \geq 1-\beta$.
This is done by reducing the estimation problem to \autoref{testing}.

\begin{problem}
\label{testing}
Let $\nu$ be a distribution over $\Real$.
Define the following null and planted distributions: 
\begin{itemize}
\item 
Under $\cQ$, observe i.i.d.\ samples $\bX_1,\ldots, \bX_n \sim \cN(0,\bI_d)^n$. 
\item 
Under $\cP$, first draw $\bu^*$ uniformly from $\bbS_{d-1}$ and i.i.d. $\bnu_1, \ldots, \bnu_n$. Conditional on $\bu^*$ and $\{\bx_i\}$, draw independent samples $\bX_1,\ldots,\bX_n \in \Real^d$ where $\bX_i \sim \cN(\bnu_i \bu^*, \bI_n - \bu^*{\bu^*}^\top)$.
Note that this is equivalent to $\cD_{\cF}$ in \autoref{planted_vector_setting}.
\end{itemize}
Suppose that we observe the matrix $\bX \in \Real^{n \times d}$ with rows $\bX_1^\top, \dots, \bX_n^\top$. We aim to test between the hypotheses $\cQ$ and $\cP$. 
\end{problem}

In the following, we will show evidence of a computational statistical gap in the Planted Vector Setting \autoref{planted_vector_setting}.
Specifically, we will demonstrate that for $n = \widetilde{\cO}(d^{1.5} p)$ the Low Degree Likelihood ratio stays bounded and thus, according to the \conjectureRef\hspace{-0.1cm}, no polynomial time algorithm can test \autoref{testing}.
\begin{theorem}
\label{thmLDPLB}
For an instance of \autoref{testing} with $\bnu \sim \cB(p)^n$ and $n = \widetilde{\cO}(d^{1.5} p)$ the Low Degree Likelihood Ratio stays bounded for degree $D = \log(d)^{1+\epsilon}$ for $\epsilon > 0$.
$${\ldlr^2} \leq 2.$$
\end{theorem}

Finally, we reduce the estimation problem to \autoref{testing}.
To do this, we have to demonstrate that it is possible to construct a test for \autoref{testing} if we have access to an estimate $\hat{\bu}$ for $\bu^*$ such that $\inner{\hat{\bu}, \bu^*} \geq 1 - \beta$.
\autoref{reduction} shows it is possible to construct such a test for \autoref{testing} if $n = \Omega(d)$.
\begin{corollary}
    \label{reduction}
    For all $\hat{\bu}$ for which $\inner{\hat{\bu}, \bu^*} \geq 1 - \beta$ for sufficiently small $\beta > 0$ and $\frac{1}{\sqrt{d}} \leq p < \frac{1}{2}$.
    Define the test $\Psi$:
    \vspace{-0.1cm}
    \begin{align*}
        \Psi :=
        \begin{cases}
            \cQ & \, \text{if} \, \sumin \frac{\phi(\inner{\bX_i, \hat{\bu}})}{n} < t \\
            \cP &\, \text{if} \, \sumin \frac{\phi(\inner{\bX_i, \hat{\bu}})}{n} \geq t
        \end{cases}
    \end{align*}
    \vspace{-0.1cm}
    With $\phi(x) = \max\{0, x\}^2$.
    There exists a threshold $t$ such that
    \vspace{-0.1cm}
    $$
        \bbP_{\cQ}\left\{ \Psi = \cP \right\}
        +
        \bbP_{\cP}\left\{ \Psi = \cQ \right\}
        \leq 
        \exp \left( - \bigTheta{\frac{n}{d}} \right)
    $$
\end{corollary}

In the regime where $n < d+1$, we refer to \citet{zadik2022latticebased}, demonstrating the statistical impossibility of estimation in this regime.
Thus, combining \autoref{reduction} and \autoref{thmLDPLB} yields the result that if the \conjectureRef is true, no polynomial time can estimate the planted vector.
To our knowledge, currently, no efficient algorithm exists that can recover the signal direction with $n = o(d^2 p^2)$ samples.
\begin{remark}
    \citet{dudeja2024statisticalcomputational} gives results on the failure of Low-Degree Polynomials in the Planted Vector Setting.
    If for $i \in \{1, \ldots, k-1\}$ the moments of $\nu$ and the standard normal distribution match a bound of $n \ll d^{k/2} \lambda^{-2}$, where $\lambda$ is the signal to noise ratio defined as 
    $
        \left|\bbE[\nu^k] - \bbE[Z^k] \right| = \lambda \quad Z \sim \cN(0,1)
    $.
    In the case of $\nu \sim \cB(p)$ for $k = 3$ we obtain $\lambda \geq \frac{2}{\sqrt{p}}$.
    Here, we note that the bound is not applicable to our setting as by choosing $p$ sufficiently small, Assumption 2 cannot be fulfilled anymore.
    Thus, for completeness, we give the same bound of $n = \widetilde{\cO}(d^{1.5} p)$ which is valid for $p \in \left(\frac{1}{\sqrt{p}}, \frac{1}{2} \right)$.
\end{remark}
\subsection{Bernoulli Rademacher Planted Vectors}
The Bernoulli Rademacher setting has been thoroughly studied in \citet{mao2022optimal}.
Here, the failure of Low Degree Polynomials when $n = \widetilde{O}(d^{2} p^2)$ is demonstrated.
This lower bound is known to be tight, as a spectral algorithm can recover a Bernoulli-Rademacher planted vector with $n = \widetilde{\Theta}(d^{2} p^2)$ samples, which is tight, as there exist spectral methods which can recover the planted vector when $n = \widetilde{\Omega}(d^2 p^2)$ samples~\cite{hopkins2016fast, mao2022optimal}.
Our gradient-based method has a sample complexity of $n = \widetilde{\cO}(d^3 p^4)$, which is larger than what can be achieved using spectral methods.
In the case where $p = \frac{1}{\sqrt{d}}$, this matches the bounds obtained using spectral methods.
\begin{remark}
\citet{zadik2022latticebased} gives an algorithm to recover planted vectors using LLL-basis reduction.
This algorithm does not exhibit the statistical-to-computational gap.
This algorithm is only applicable in a very restrictive setting, as it is required that the planted vector can only take on a set of discrete values.
This can be avoided by considering a setting with a small amount of noise.
E.g., considering a hierarchical setting where $\cB(p')$ with $p' \sim \cU([\frac{p}{2}, p])$ is a simple counterexample, in which our analysis still works and where the algorithm discussed in \citet{zadik2022latticebased} fails due to the lack of robustness to noise.
\end{remark}  
\section{Experiments}
\subsection{Experiments with Synthetic Data}
\label{simulations}
In the following, we will validate the findings in \autoref{relu2_sc} and \autoref{kurtosis_sc} by running \autoref{twoStepAlgorithm} on synthetic datasets.
In the Imbalanced Clusters setting, we will be choosing the dimension $d \in \{16, ..., 512\}$ and the cluster imbalance for $p \in \{d^{-0.5}, d^{-0.3}, 0.3\}$ with $n \in \{16, ..., 2048\}$.
The algorithm is executed for $s = 2 \log_2 d$ steps and $\eta_1 = \sqrt{d} p$ and $\eta_2 = 0.5$.
The results are plotted in \autoref{asymptotics_experiment}.
In the Bernoulli-Rademacher setting we will be choosing the dimension $d \in \{16, ..., 256]$ and the cluster imbalance as either $p \in \{d^{-0.5}, d^{-0.5}, 0.3\}$ with $n \in \{16, ..., 4096\}$.
The algorithm is executed for $s = 2 \log_2 d$ steps and $\eta_1 = \sqrt{d} p$ and $\eta_2 = 0.5$.
The results are plotted in \autoref{asymptotics_experiment_kurtosis}.
To evaluate, we plot the average value of $\inner{\hat{\bu}, \bu^*}$ over $30$ independently sampled datasets for each $d, n, p$.
It can be observed that for sufficiently large $d$, the asymptotics closely match the experimental sample complexity.
\begin{figure}[H]
    \vspace{-0.3cm}
    \centering
    \begin{subfigure}[b]{0.8\textwidth}
        \centering
        \vspace{-0.3cm}
        \includegraphics[scale = 0.4]{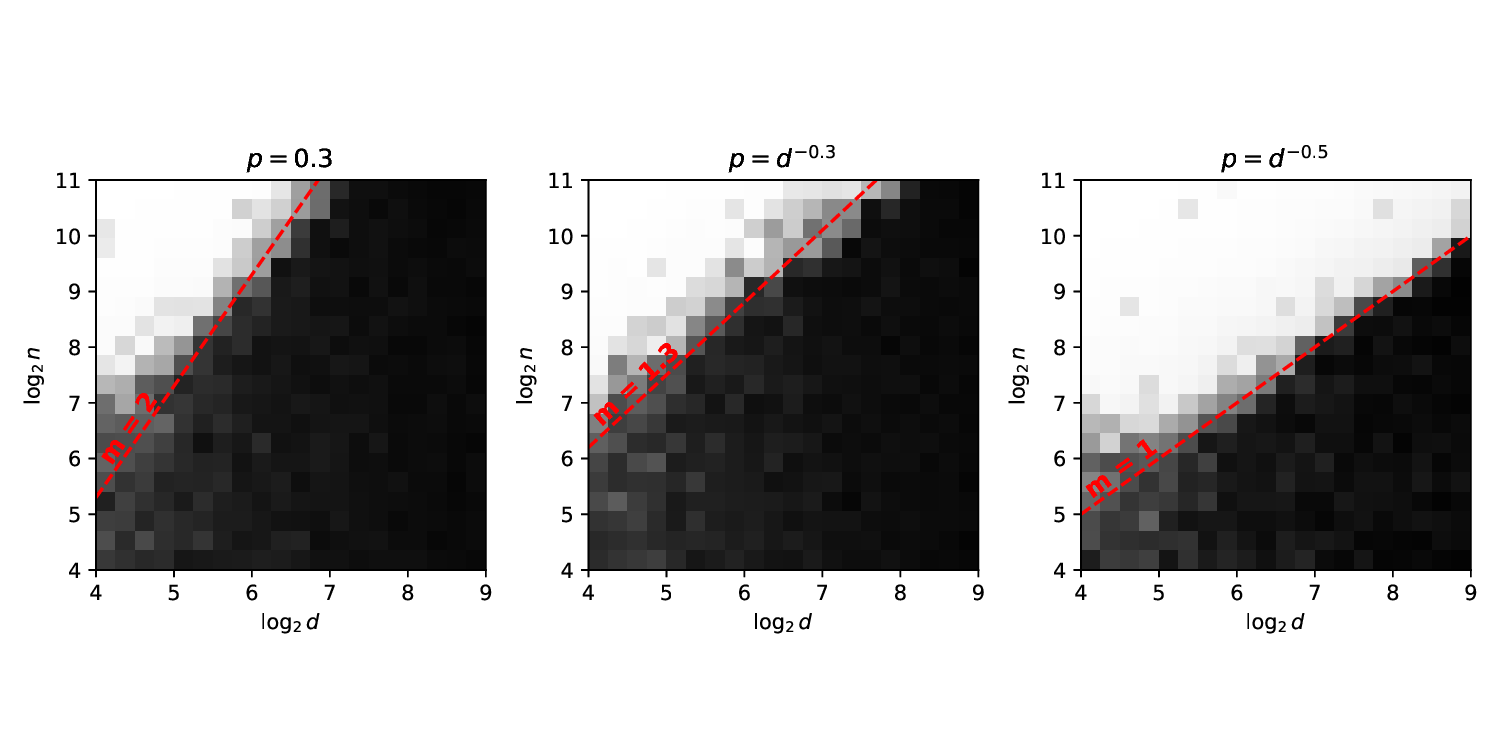}
        \vspace{-0.7cm}
        \caption{
            Projection pursuit of Imbalanced Clusters using \autoref{twoStepAlgorithm} with $\phi(x) = \psi(x) = \max\{0, x\}^2$.
            The red lines of slopes $(2, 1.3, 1)$ roughly highlight the phase transition.
        }
        \label{asymptotics_experiment}
    \end{subfigure}
    \begin{subfigure}[b]{0.8\textwidth}
        \centering
        \vspace{-0.1cm}
        \includegraphics[scale = 0.4]{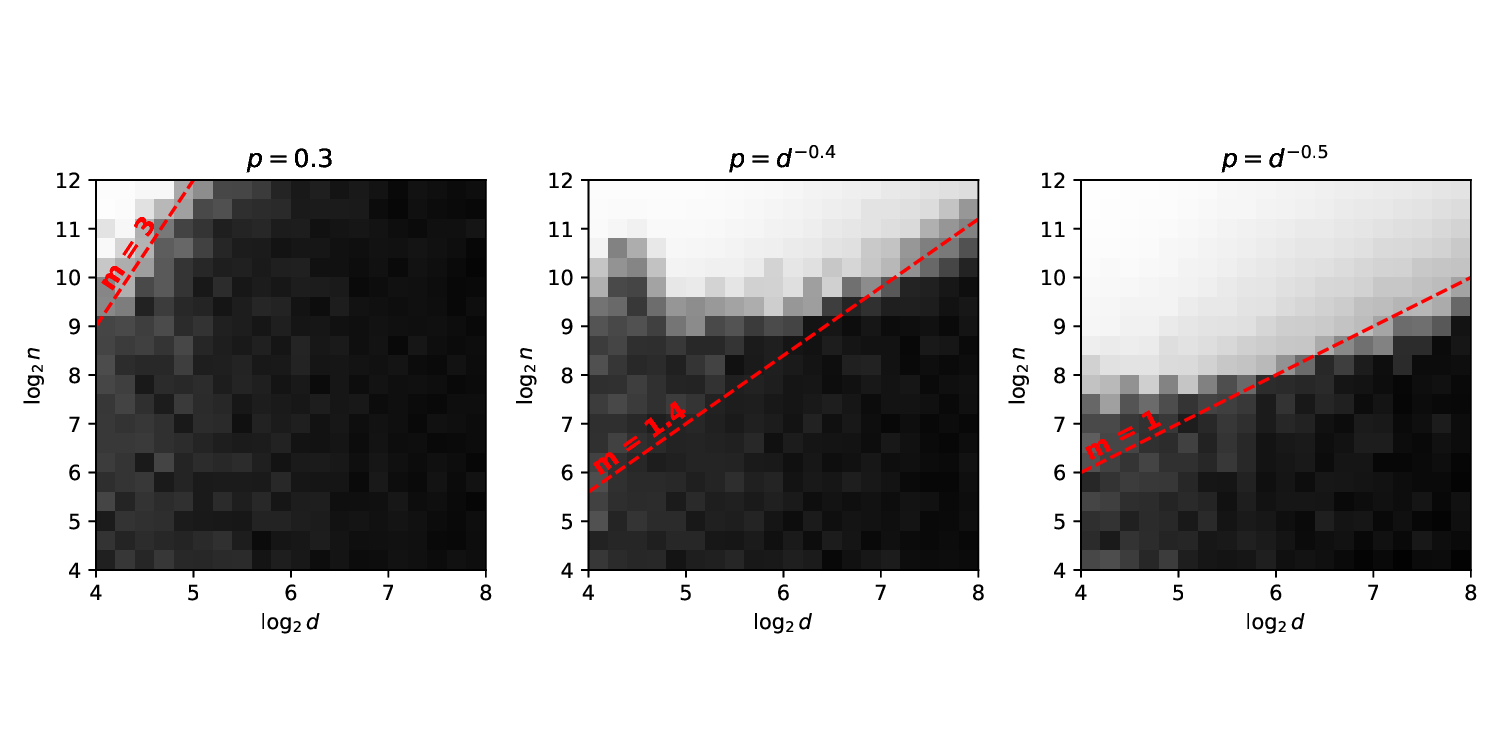}
        \vspace{-0.7cm}
        \caption{
            Projection pursuit of a sparse Bernoulli-Rademacher distribution using \autoref{twoStepAlgorithm} with $\phi(x) = x^4$ and $\psi(x) = -|x|$.
            The red lines of slopes $(3, 1.8, 1)$ roughly highlight the phase transition.
            }
        \label{asymptotics_experiment_kurtosis}
    \end{subfigure}
    \caption{
        Phase transitions in Projection Pursuit using gradient-based algorithms.
        The horizontal and vertical axes correspond to $\log_2 d$ and $\log_2 n$.
        Each pixel shows the average value of the absolute inner product between the predicted and signal directions, where white corresponds to 1 and black to 0.
    }
\end{figure}
\subsection{Comparison to Other Methods}
\label{discussion}
Here, we will apply \autoref{twoStepAlgorithm} to the planted vector problem with other projection indices.
Specifically, we will consider a list of projection pursuit methods as presented in \autoref{fig:pp_inidces}.
\begin{figure}[H]
    \centering
    \renewcommand{\arraystretch}{1.3}
    \begin{tabular}{ r | l | l}
        \textbf{Abbreviation} & $\sumin \frac{\phi(\bX_i)}{n}$ & $\sumin \frac{\psi(\bX_i)}{n}$ \\
        \hline & \\[-2ex]
        ReLU2 & $\average{\max\{0, \bX_i\}^2}$ & $\average{\max\{0, \bX_i\}^2}$\\
        Kurtosis & $\average{\bX_i^4}$ & $\average{-|\bX_i|}$\\
        Abs & $\average{-|\bX_i|}$ & $\average{-|\bX_i|}$\\
        AbsMax & $\average{|\bX_i|}$ & $\average{|\bX_i|}$\\
        Skewness & $\average{\bX_i^3}$ & $\average{\bX_i^3}$\\
        ApproxEntropy & $(\average{\bX_i^3})^2  + (\average{\bX_i^4} - 3)^2$ & $(\average{\bX_i^3})^2 + (\average{\bX_i^4} - 3)^2$\\
    \end{tabular}
    \vspace{0.5cm}
    
    \begin{tabular}{ r | l }
        \textbf{Abbreviation} & Method \\
        \hline & \\[-2ex]
        Cov4max & $\mathbf{v}_{\lambda_{\max}}$ $(\sumin \mynorm{\bX_i}^2 \bX_i \bX_i^\top)$ \\
        Cov4min & $\mathbf{v}_{\lambda_{\min}}(\sumin \mynorm{\bX_i}^2 \bX_i \bX_i^\top)$ \\
        3TensorDecomp & $\mathbf{v}_{\sigma_{\min}}(\sumin \bX_i \otimes \bX_i \bX_i^\top)$ \\
    \end{tabular}
    \caption{
        Methods used in the empirical comparison.
        On the left: Methods using the gradient-based algorithm.
        On the right: Spectral methods.
    }
    \label{fig:pp_inidces}
\end{figure}
ReLU2 and Kurtosis correspond to the projection indices we study in \autoref{relu2_sc} and \autoref{kurtosis_sc}.
Additionally, we test the projection indices Abs, which corresponds to the projection index used in \citet{QuSW14}, AbsMax, which corresponds to the objective in \citet{davis2021clustering} and Skewness as proposed in \citet{skewness2004paajarvi}.
ApproxEntropy uses the entropy approximation as introduced by \citet{NIPS1997_6d9c547c}.
We choose this selection of projection indices as a representative sample of projection indices studied in the literature. 
We choose projection indices that consist of basic functions as compared to more complex projection indices (e.g., \citet{PPFriedmanTukey}) to ensure that they work well with gradient-based optimization.

Additionally, we compare the gradient-based methods to three spectral methods.
\citet{nicolaSkewness} introduces a method called MaxSkew for finding directions of large skewness(abbreviated by 3TensorDecomp).
Cov4max and Cov4min denote the eigenvectors corresponding to the largest and smallest eigenvalue of a kurtosis matrix as introduced in \citet{Pena2010} and \citet{Kollo2008}.
This kurtosis matrix is also studied in the context of sample efficient recovery of planted vectors \cite{mao2022optimal} and clustering of Gaussian mixtures \cite{davis2021clustering}.

Fully resampling mini-batches seems to be unnecessary if the dimension is sufficiently low.
Thus, it tends to be beneficial to subsample the dataset with replacement, which we will do for the following experiments.
Here, we compare the previously mentioned methods in the planted vector setting with a Bernoulli Rademacher planted vector in \autoref{br_comp} and with an Imbalanced Clusters planted vector in \autoref{ic_comp} with $d = 300$, $p = 0.1$.
For \autoref{twoStepAlgorithm} we use $n_{init} = 400$ initializations.
In \autoref{fig:spectral_methods_comparison}, we can observe that most algorithms only perform well on one of both settings. 
In \autoref{br_comp}, we observe that Cov4max performs best closely followed gradient-based projection pursuit using Kurtosis.
This behavior is as expected by bounds on the sample complexity in the Bernoulli Rademacher setting, as Cov4max only needs $\bigOmegas{d^2 p^2}$ samples to recover the signal direction~\cite{mao2022optimal}.
In the planted vector setting with Imbalanced Clusters ReLU2 performs best.
\begin{figure}[H]
    \vspace{-0.3cm}
    \centering
    \begin{subfigure}[b]{0.49\textwidth}
        \includegraphics[scale = 0.4]{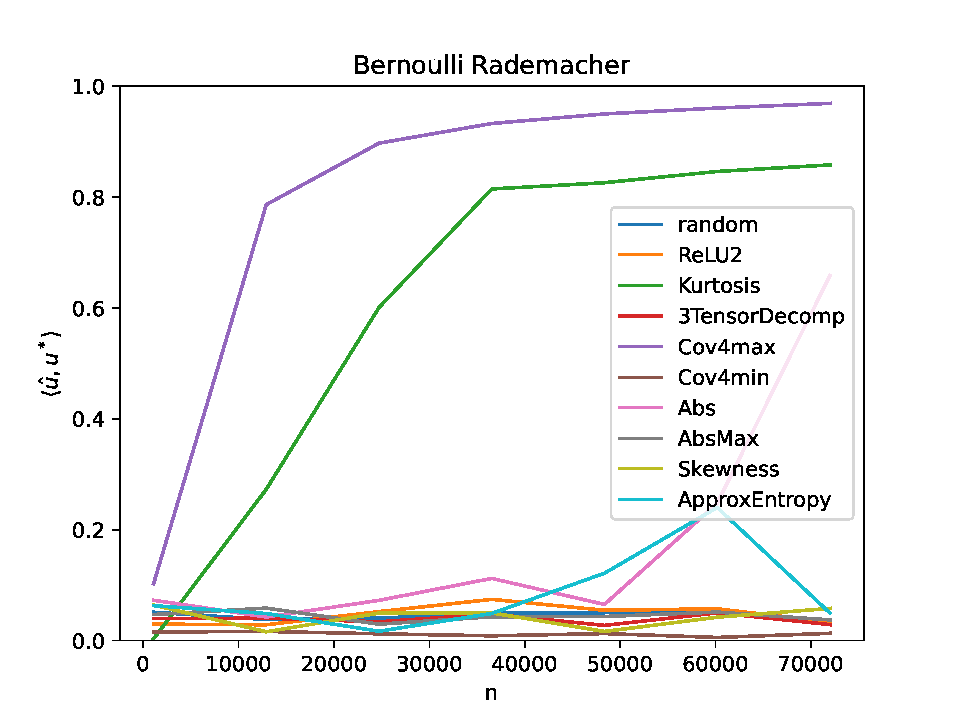}
        \caption{Bernoulli-Rademacher}
        \label{br_comp}
    \end{subfigure}
    \hfill
    \begin{subfigure}[b]{0.49\textwidth}
        \includegraphics[scale = 0.4]{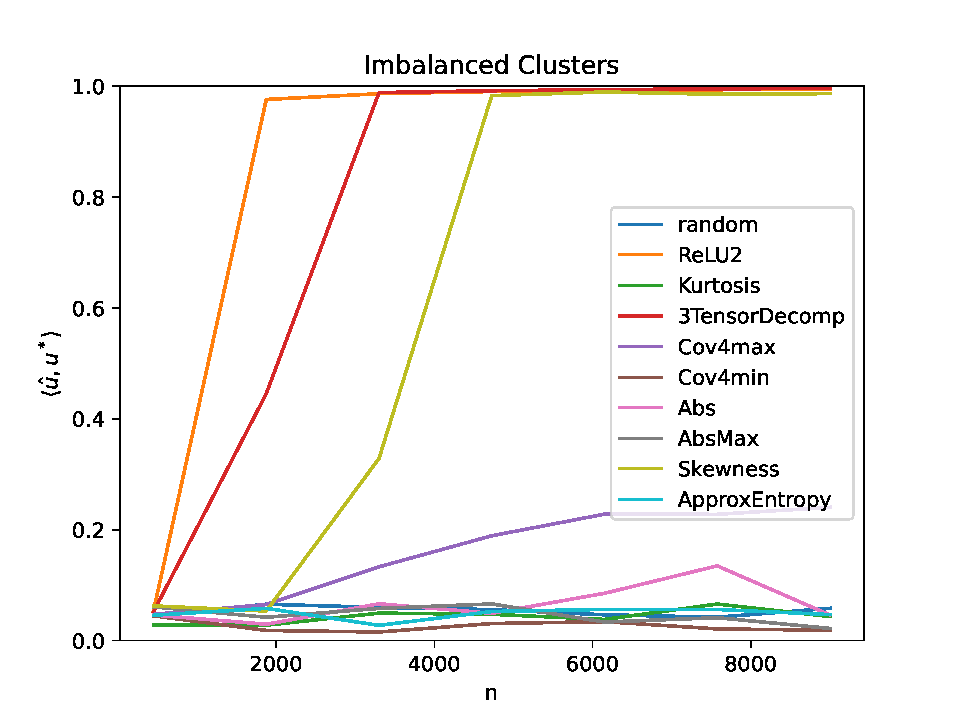}
        \caption{Imbalanced Clusters}
        \label{ic_comp}
    \end{subfigure}
    \caption{
        Comparison of different methods in the planted vector setting.
        We plot the average inner product between the signal direction and the recovered direction by each algorithm over 30 datasets.
    }
    \label{fig:spectral_methods_comparison}
\end{figure}
\subsection{Experiments with Real Data}
\label{experiments}
We compare the algorithms by measuring their performance on two different datasets: FashionMNIST~\cite{xiao2017fashionmnist} and the Human Activity Recognition Dataset~\cite{human_activity_recognition_using_smartphones_240}.
The FashionMNIST dataset contains $28 \times 28$ gray-scale images of fashion articles with articles labeled into $10$ categories denoting the article type (T-shirt/top, Trouser, Pullover, Dress, Coat, Sandal, Shirt, Sneaker, Bag, Ankle boot).
The Human Activity Recognition Dataset~\cite{human_activity_recognition_using_smartphones_240} has been collected from 30 subjects performing daily tasks while carrying a smartphone for inertial measurements.
Each sample consists of 561 summary statistics of the inertial measurements and is labeled by the performed activity (Walking, Walking Upstairs, Walking Downstairs, Sitting, Standing, Lying) by the subject.

For the experiments, we reduce the dimension to $100$ using PCA.
We use $n_{init} = 500$ initializations for the gradient-based algorithm and choose the $30$ directions with the largest value for the projection index.
For spectral methods, we run the spectral method $30$ times while removing the recovered directions from the dataset.
To compare the performance of the different methods, we evaluate how well a single projection can help predict the labels of the images.
This will be evaluated using the Information Gain $\mathrm{IG}(Y,A)=\mathrm {H}(Y)-\mathrm {H}(Y|A)$.
Where $Y$ are the labels assigned to the instances in the respective datasets.
Here, we choose the indicator $A = \bbI_{\inner{\bu, \bx} > t}$, where $t$ is an automatically chosen threshold for each index to maximize information gain on the training dataset.
The final information gain is evaluated on a holdout dataset.
We evaluate the projection pursuit methods on both a very small and a very large training set to determine which methods perform well with only a small number of samples and which need a large number of samples.

In \autoref{fig:InformationGain}, we plot the results for the FashionMNIST dataset, and in \autoref{fig:InformationGain_HAR}, we plot the results for the Human Activity Recognition Dataset.
Here, we can observe that the ReLU2 projection index still performs well even if only a few samples are present, outperforming all other methods.
Other methods such as AbsMax can only be optimized with significantly more samples but perform significantly better if optimization is successful than other projection pursuit methods.
Additionally, we note that using Kurtosis seems to perform very similarly to Abs, which is to be expected due to both algorithms using the same projection index to select projections.
We hypothesize that this is because the gradient of Kurtosis is less stable than other projection indices such as Abs offsetting the faster convergence of Kurtosis, which seems to be less important in lower dimensional data.
\begin{figure}[H]
    \vspace{-0.3cm}
    \centering
        \begin{subfigure}[b]{0.3\textwidth}
            \centering
            \includegraphics[width=\textwidth]{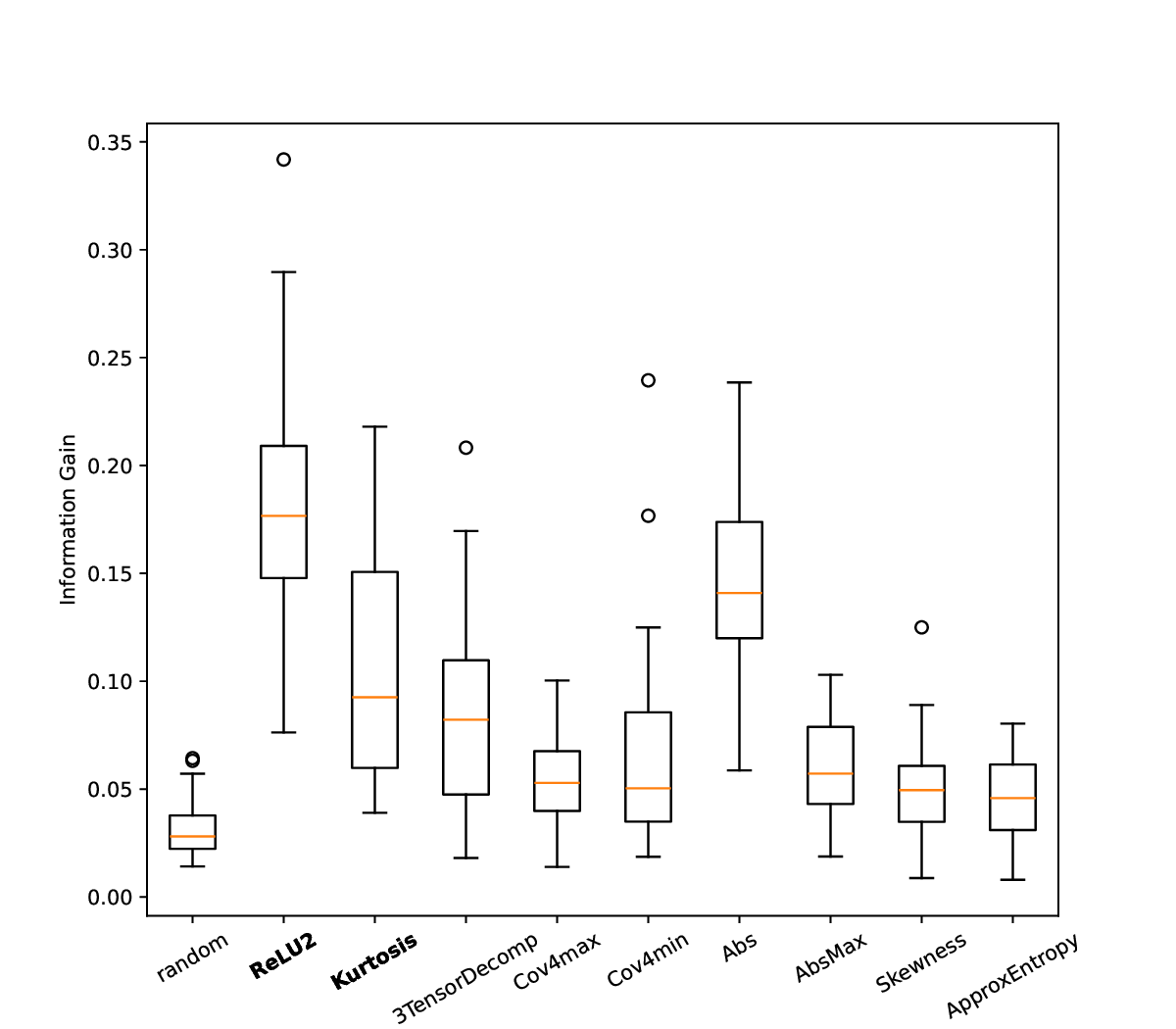}
            \caption{$n = 600$}
        \end{subfigure}
        \begin{subfigure}[b]{0.3\textwidth}
            \centering
            \includegraphics[width=\textwidth]{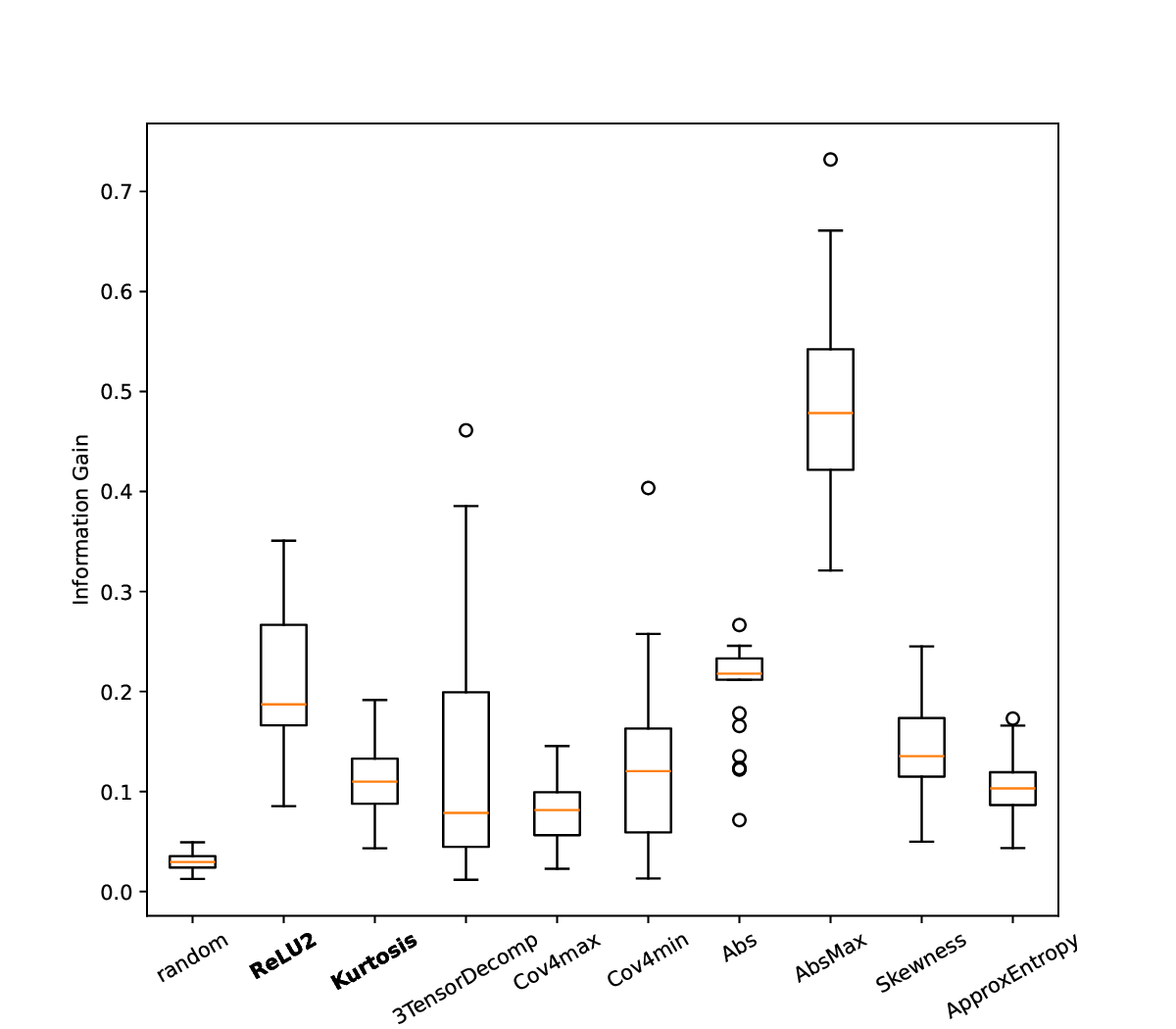}
            \caption{$n = 60000$}
        \end{subfigure}
    \caption{
        A comparison of projection pursuit approaches on fashionMNIST.
        We plot the achieved information gain using projections produced by different projection indices.
        The box plots are generated with the 30 candidate projections generated by the algorithms.
    }
    \label{fig:InformationGain}
\end{figure}
\begin{figure}
    \vspace{-0.3cm}
    \centering
        \begin{subfigure}[b]{0.3\textwidth}
            \centering
            \includegraphics[width=\textwidth]{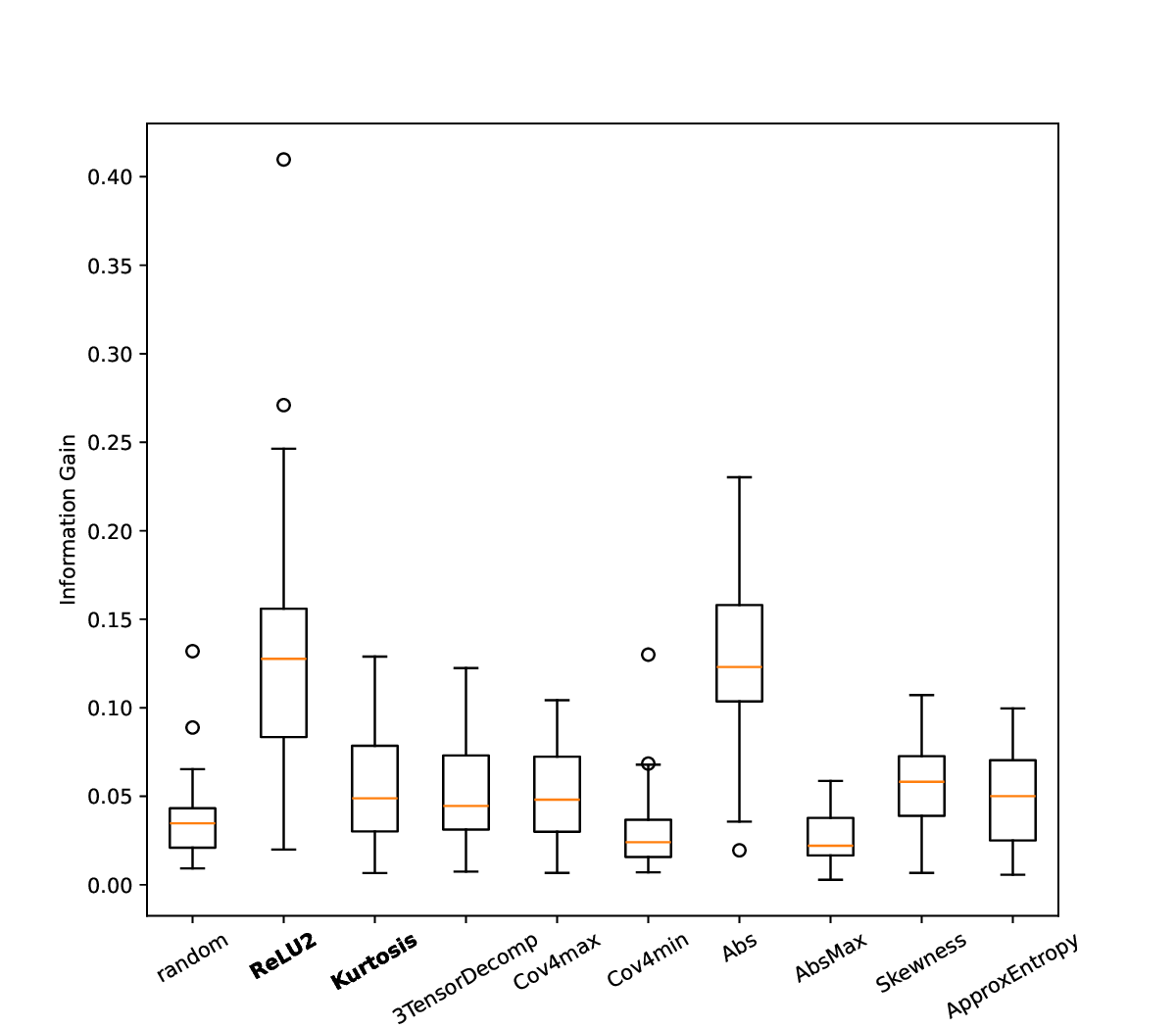}
            \caption{$n = 300$}
        \end{subfigure}
        \begin{subfigure}[b]{0.3\textwidth}
            \centering
            \includegraphics[width=\textwidth]{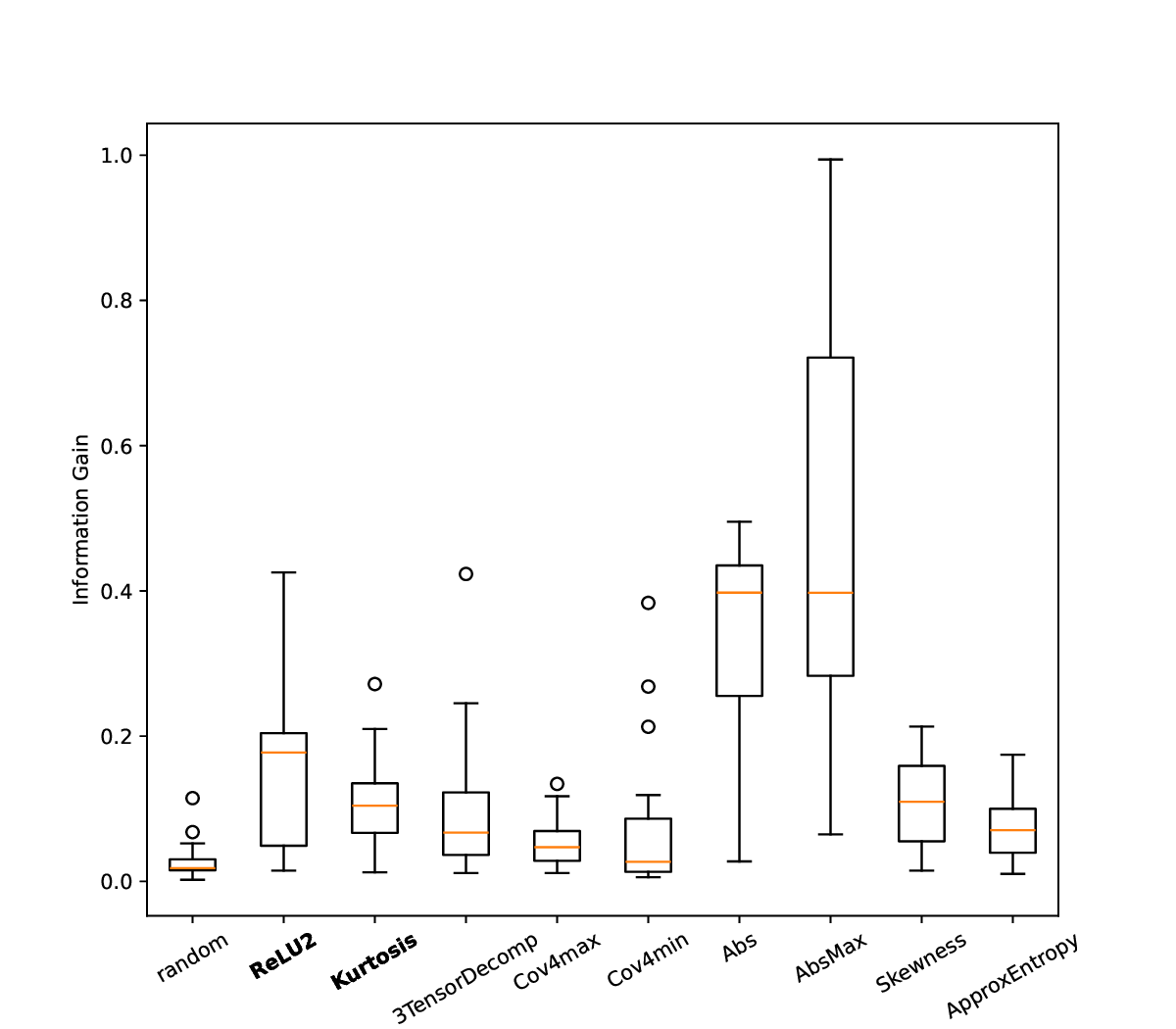}
            \caption{$n = 7352$}
        \end{subfigure}
    \caption{
        A comparison of projection pursuit approaches on the Human Activity Recognition using Smartphones Benchmark.
        We plot the achieved information gain using projections produced by different projection indices.
        The box plots are generated with the 30 candidate projections generated by the algorithms.
    }
    \label{fig:InformationGain_HAR}
\end{figure}
For demonstration purposes, we also show histograms of the data projected onto the recovered projections in \autoref{fig:histograms_fmnist_r2} and \autoref{fig:histograms_fmnist_kurtosis}.
In \autoref{fig:histograms_fmnist_r2}, it can be observed that the found projections reveal two Imbalanced Clusters where the smaller cluster contains samples of mostly one class. 
\autoref{fig:histograms_fmnist_kurtosis} shows that the Kurtosis projection index recovers projections for which many samples are projected close to zero while a few are projected away from zero, similar to a Bernoulli-Rademacher distribution.
\begin{figure}
    \vspace{-0.3cm}
    \centering
    \begin{subfigure}[b]{\textwidth}
    \centering
        \includegraphics[scale = 0.4]{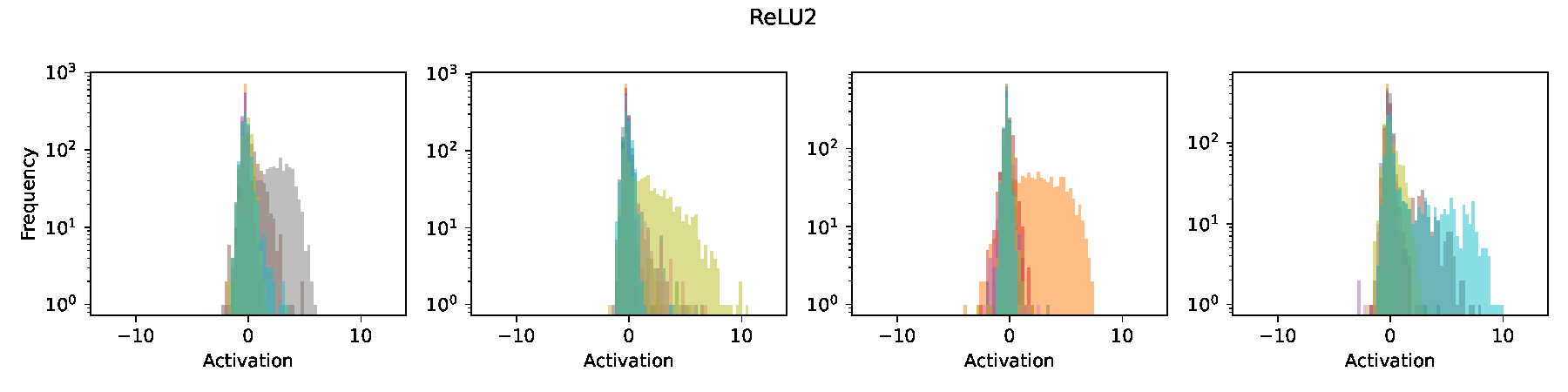}
        \caption{4 recovered projection using the ReLU2 projection index}
        \label{fig:histograms_fmnist_r2}
    \end{subfigure}
    \begin{subfigure}[b]{\textwidth}
        \centering
        \includegraphics[scale = 0.4]{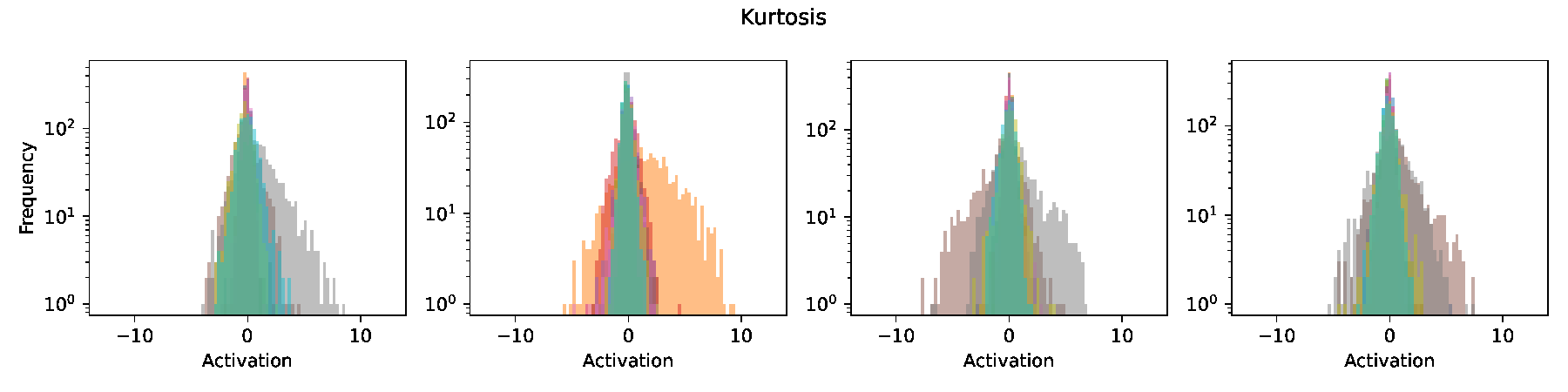}
        \caption{4 recovered projection using the Kurtosis projection index}
        \label{fig:histograms_fmnist_kurtosis}
    \end{subfigure}
    \caption{
        Histograms of projections recovered by \autoref{twoStepAlgorithm} on FashionMNIST using $n = 600$ samples.
        Classes are plotted using different colors.
    }
    \label{hists}
\end{figure}
\autoref{fig:histograms_fmnist_r2} provides insight into why discovering imbalanced projections effectively recovers label-separating projections in classification tasks such as Fashion-MNIST. 
We can reasonably assume that the data of a classification task follows a cluster structure.
If the data is projected in the direction of one cluster center, we expect that the samples from the other clusters will collapse close to $0$, while the samples from the chosen cluster will move out to one side.
This effect can also be observed in \autoref{fig:histograms_fmnist_r2}.
Consequently, projecting the data along the direction of a particular cluster center naturally yields an imbalanced histogram, with samples from the corresponding cluster positioned towards one extreme.
\section{Conclusion}
We consider the performance of gradient-based algorithms for projection pursuit in the planted vector setting, where we study their sample complexity.
Specifically, we consider the setting where the planted vector follows a distribution containing two clusters of imbalanced size or a Bernoulli-Rademacher distribution.
In the former setting, Low-Degree Polynomials give a lower bound of $n = \widetilde{\Theta}(d^{1.5} p)$ and gradient-based methods can recover the signal direction provably with $n = \widetilde{\Omega}(d^2 p^2)$ thus presenting a gap of a factor of $p \sqrt{d}$ which increases as $p$ increases.
It is currently unknown whether an algorithm exists that matches the sample complexity of the lower bound in the Low-Degree Polynomial Framework.
In the latter setting, $n = \widetilde{\Omega}(d^3 p^4)$ samples are sufficient, and there exist spectral algorithms matching the computational lower bounds of $n = \widetilde{\Omega}(d^2 p^2)$ samples.
Although there still exists a gap between gradient-based methods and computational lower bounds, we can observe that in both settings, if the distribution is very imbalanced/sparse, gradient-based methods match computational lower bounds closely.
Finally, we demonstrate the favorable performance of our algorithm if the number of samples is severely limited.
\section{Acknowledgements}
This work has been supported by the German Research Foundation (DFG) through DFG-ANR PRCI ``ASCAI" (GH 257/3-1).
\appendix
\bibliographystyle{myjmva}
\bibliography{bib}
\section{Proof of \autoref{mainTheorem}}
\begin{proof}   
\label{proofMainTheorem}
Start by choosing an initialization $j \in [n_{init}]$ for which $\inner{\bu_{0, j}, \bu^*} \geq a$ which can be assumed to exist with probability at least $1 - \smallO{1}$.
By \autoref{assGradient} we know that for each step $\inner{\bu_{t+1, j}, \bu^*} \geq (1 + c_0) \inner{\bu_{t, j}, \bu^*}$ if $\inner{\bu_{t, j}, \bu^*} < b$ with proabability at least $1 - \smallO{\frac{1}{s}}$.
By applying the union bound over all steps, we can obtain that
$s = \log_{1 + c_0}\left( \frac{b}{a} \right) = \cO(\log(d))$ steps are sufficient, such that at least one direction $\bu_{i,j}$ is encountered for which $\inner{\bu_{i,j}, \bu^*} \geq b$.
By \autoref{assTestability} we can observe that $\hat{\bu} = \argmax_{\bu \in \{\bu_{i,j} | i \in [s], j \in [n_{init}]\}} \sum_{k = 1}^{n} \frac{\phi(\inner{\bX_k, \bu})}{n}$ has $\inner{\hat{\bu}, \bu^*} \geq b - \delta$ finishing the proof.
\end{proof}
\vspace{-0.1cm} 
\section{Proofs in \autoref{Relu2Section}}
\label{relu2_sc_proof}
In this section use $\phi(x) = \psi(x) = \max\{0, x\}^2$.
For simplicity we also define $\mu_1 = \sqrt{\frac{1-p}{p}}$ and $\mu_2 = - \sqrt{\frac{p}{1-p}}$.
\begin{lemma}
    \label{relu2Init}
    Given $\bu \sim \Data_{\cB(p)}$ for $p \in (1/\sqrt{d}, 1/2)$
    with probability at least $\bigTheta{p}$
    $$
    \frac{\ainner}{\mynorm{\bu}} \geq \bigTheta{\frac{1}{\sqrt{p d}}}.
    $$
\end{lemma}
\begin{proof}
    We know that $\ainner = \sqrt{\frac{1-p}{p}} \geq \sqrt{\frac{1}{2 p}}$ with probability $p$.
    Conditioned on $\inner{\bu^*, \bu} = \sqrt{\frac{1-p}{p}}$ we have $\bu \sim \mathcal{N}\left( \sqrt{\frac{1-p}{p}} \bu^*, \bI_n - \bu^*{\bu^*}^\top \right)$.
    By Markov's inequality, we have
    $\mynorm{\correctiontwo \bu} \leq \bigTheta{\sqrt{d}}$ 
    with constant probability.
    The lemma follows from the combination of the previous two results.
\end{proof}
\begin{lemma}
    \label{relu2Concentration}
    For any $\delta \in (0, 1)$ and $\bX \sim \cD^n_{\cB(p)}$ have
    \begin{flalign*}
    \mathbb{P}
    \left[
    \myabs{
    \inner{\sumin \frac{g_\bu(\bX_i)}{n}, \bu^*}
    - 
    \inner{\expectation\left[g_\bu(\bX)\right], \bu^*}
    }
    \geq \delta
    \right]
    \leq
    \exp
    \left(
    - \bigTheta{n \delta^2 \min\left\{ 1, \frac{p}{\ainner^2} \right\}}
    \right).
    \end{flalign*}
\end{lemma}
\begin{proof}
For simplicity we will split $\inner{g_\bu(\bx), \bu^*}$ into terms $t_1, t_2$ such that
$$
    \inner{g_\bu(\bx), \bu^*}
    = 
    \underbrace{\max \{0, \inner{\bx, \bu}\} \inner{\bx, \bu^*}}_{= t_1(\bx)}
    - 
    \underbrace{\max \{0, \inner{\bx, \bu}\} \ainner \inner{\bx, \bu}}_{= t_2(\bx)}
    .
$$
Here we use the Sub-Gaussian norm $\| \cdot \|_{\psi_2}$ and Sub-Exponential norm $\| \cdot \|_{\psi_1}$ as defined in \cite{vershynin2018high}.

Next we will choose a fixed sample $\hat{\nu} \sim \cB(p)$ such that $\inner{\bX_i, \bu^*} = \hat{\nu}_i$ and define the \emph{empirical} balance as $\hat{p} = \sumin \frac{\indicator_{\hat{\nu} = \mu_1}}{n}$.
For now we will assume $(1 - \bar{\delta}) p \leq \hat{p} \leq (1 + \bar{\delta}) p$.

Bounding $\sumin t_1(\bX_i)$:
First note that if the random variable $Z$ is Sub-Gaussian, then there exists a constant $C_{centering}$ such that $\twoOrclizNorm{Z - \expectation[Z]} \leq C_{centering} \twoOrclizNorm{Z}$.
Also note that $\|\max\{ c, Z \}\|_{\psi_2} \leq \twoOrclizNorm{Z}$.

Thus 
$
\sum_{i = 1}^n 
\twoOrclizNorm{
    \frac{t_1(\bX_i)}{n}
- 
\expectation
\left[
    \frac{t_1(\bX_i)}{n}
\right]
}^2
\leq 
\hat{p} \frac{\mu_1^2}{n}
+ 
(1 - \hat{p}) \frac{\mu_2^2}{n}
\stackrel{(1)}{\leq}
\bigTheta{\frac{1}{n}}
$.
With (1) following from $\hat{p} \leq (1 + \bar{\delta}) p$.
Thus 
$$
\mathbb{P}\left[ 
\myabs{
    \sum_{i = 1}^n \frac{t_1(\bX_i)}{n}
- 
\expectation
\left[
    \sum_{i = 1}^n \frac{t_1(\bX_i)}{n}
\right]
} \leq \delta \right] 
\geq 1 - \exp\left( - \bigTheta{n \delta^2} \right)
$$

Bounding $\sumin t_2(\bX_i)$:
Note that if the random variable $Z$ is Sub-Exponential then there exists a constant $C_{centering}$ such that $\|Z - \expectation[Z]\|_{\psi_1} \leq C_{centering} \| Z \|_{\psi_1}$.
Additionally note if  $\|Z - \expectation[Z] \|_{\psi_2} = K$ and $\expectation[Z] < 0$ then $\|\max\{0, Z\}\|_{\psi_2} \leq K$.

Thus
$\sum_{i = 1}^n \oneOrclizNorm{
\frac{t_2(\bX_i)}{n}
- 
\expectation
\left[
    \sum_{i = 1}^n \frac{t_2(\bX_i)}{n}
\right]
}^2 
\leq 
\hat{p} \frac{\mu_1^4}{n^3}
+ 
(1 - \hat{p}) \frac{\mu_2^4}{n^3}
\leq
\bigTheta{\frac{1}{n^2}}
$ 
for $n > d$
and $\max_{i \in [n]} \| t_2(\bX_i) \|_{\psi_1} \leq \bigTheta{\frac{1}{n}}$.
Thus, we can show the following bound by applying Bernstein's inequality for sufficiently large $n \geq \frac{1}{\delta}$
\begin{flalign*}  
\mathbb{P}
\left[
\myabs{
    \sum_{i = 1}^n \frac{t_2(\bX_i)}{n}
- 
\expectation
\left[
    \sum_{i = 1}^n \frac{t_2(\bX_i)}{n}
\right]
}
\leq \delta
\right]
& \geq 
1 - \exp\left(
- \bigTheta{
    \min\left\{
    \frac{\delta^2}{\sum_{i = 1}^n \| t_2(\bX_i) \|_{\psi_1}^2},
    \frac{\delta}{\max_{i \in [n]} \| t_2(\bX_i) \|_{\psi_1}}
    \right\} 
}
\right) \geq 
1 - \exp\left\{- \bigTheta{n \delta} \right\}.
\end{flalign*}

Now, what is left to do is to bound $\projectustar{g_\bu(\bx)}$ with a change in $\hat{p}$.
$$
    \myabs{
    \expectation \left[ \sumin \inner{g_\bu(\bx), \bu^*} \right] - 
    \expectation \left[ \sumin \inner{g_\bu(\bx), \bu^*} \bigg| \hat{p} = (1 - \bar{\delta}) p\right]} 
    \leq \bigTheta{\ainner \bar{\delta}}
$$
By applying the Chernoff bound for the Binomial distribution, we obtain for all $\bar{\delta} \in (0,1)$
$$\mathbb{P}\big[(1 - \bar{\delta}) p \leq \hat{p} \leq (1 + \bar{\delta}) p\big] \geq 1 - 2 \exp\left( - \frac{\bar{\delta}^2 n p}{3} \right).$$
The lemma follows by choosing $\bar{\delta} = \bigTheta{\frac{\delta}{\ainner}}$ and applying the union bound.
\end{proof} 

\begin{lemma}
    \label{relu2Expectation}
    For arbitrary $\beta > 0$
    there exist constants $t, p_0 > 0$ such that with $\bX \sim \cD_{\cB(p)}$ for $p < p_0$ we have\\
    for $\ainner (\mu_1 - \mu_2) < t$
    $$
        \expectation [\inner{g_{\bu}(\bX), \bu^*}] 
        \geq 
        \bigTheta{\frac{\ainner^2}{\sqrt{p}}}
    $$
    and for $\frac{t}{ (\mu_1 - \mu_2)} \leq \ainner \leq 1 - \beta$ we obtain
    $$
        \expectation [\inner{g_{\bu}(\bX_i), \bu^*}] \geq \bigTheta{\ainner}
    $$
\end{lemma}

\begin{proof}
\bgroup
\renewcommand{\ainner}{a_1}
\renewcommand{\sigmainner}{a_2}
For simplicity abbreviate $\ainner = \inner{\bu, \bu^*}$ and $\sigmainner = \sqrt{1 - \inner{\bu, \bu^*}^2}$.
Define $\be_2 = \frac{\bu - \ainner \bu^*}{\sigmainner}$.
Since $\inner{\bx, \be_2} \sim \mathcal{N}(0, 1)$ we know
$\expectation[\max \{0, \inner{\bx, \be_2}\}^2] = \frac{1}{2}$
and 
$\expectation[\max \{0, \inner{\bx, \be_2}\}] = \sqrt{\frac{1}{2 \pi}}$ by applying the expectation of the half normal distribution.
and 
$f_{\inner{\bx, \be_2}}(y) \geq \underline{p}$ for $y \in [-t_0, t_0]$.

Abbreviate $f_i(x) = f_{\inner{\bx, \be_2}}\left( \frac{x - \mu_i \ainner}{\sigmainner} \right) = f_{\mathcal{N}(0,1)}\left( \frac{x - \mu_i \ainner}{\sigmainner} \right)$.
By applying \autoref{lemProjectedGrad} have

\vspace{-0.1cm}
\begin{align*}
& \expectation_{\bx \sim \Data_p} [\inner{g_\bu(\bx), \bu^*}]
 = 
\expectation_{\bx \sim \Data_p} (\inner{\bx, \bu} \inner{\bx, \bu^*} - \inner{\bx, \bu}^2 \ainner)
 =
\int\limits_0^\infty
p (x \mu_1 - x^2 \ainner) f_1(x) 
+
(1-p) (x \mu_2 - x^2 \ainner) f_2(x)
dx
\\
& =
\int\limits_0^{(\mu_1 - \mu_2) \ainner}
p \mu_1 x f_1(x)
dx
+ 
\sqrt{(1-p) p}
\int\limits_0^\infty
\left( x + (\mu_1 - \mu_2) \ainner - x \right) f_2(x)
dx
 - 
\ainner
\int\limits_0^\infty
x^2
(p f_1(x) + (1-p) f_2(x))
dx.
\end{align*}

We continue by bounding each term.

\paragraph{Case 1} If $(\mu_1 - \mu_2) \ainner \leq t_0$
\begin{align*}
\int\limits_0^{(\mu_1 - \mu_2) \ainner}
p \mu_1 x f_1(x)
dx &
\geq 
\frac{\underline{p} p \mu_1}{2}
\left(
(\mu_1 - \mu_2) \ainner
\right)^2
\geq \frac{\underline{p}\ainner^2}{2}\left(\frac{(1-p)^{1.5}}{\sqrt{p}} + \frac{p^{1.5}}{\sqrt{1-p}} - 2 \sqrt{p (1-p)} \right)
 \geq
\bigTheta{\frac{\ainner^2}{\sqrt{p}}}.
\end{align*}

\paragraph{Case 2} If $(\mu_1 - \mu_2) \ainner > t_0$.
\vspace{-0.2cm}
\begin{align*}
p \mu_1
\int\limits_0^{(\mu_1 - \mu_2) \ainner}
x f_1(x) dx &
\geq 
\sqrt{(1-p) p}
\int\limits_{(\mu_1 - \mu_2) \ainner - t_0}^{(\mu_1 - \mu_2) \ainner} x f_1(x)
dx 
\geq 
\underline{p}
\sqrt{(1-p) p}
\left(
(\mu_1 - \mu_2) \ainner
\right)^2
- 
\left(
(\mu_1 - \mu_2) \ainner
- t_0
\right)^2
\\ & \geq 
\underline{p}
\sqrt{(1-p) p}
\left(
2
t_0
(\mu_1 - \mu_2) \ainner
- 
t_0^2
\right)
\geq 
\underline{p}
t_0
\sqrt{(1-p) p}
(\mu_1 - \mu_2) \ainner
\geq 
\underline{p}
t_0
\ainner.
\end{align*}

For $p < p_0$, the following terms can be bounded as such
\vspace{-0.1cm}
\begin{align*}
\sqrt{(1-p) p}
\int\limits_0^\infty
\left( (\mu_1 - \mu_2) \ainner \right) f_2(x)
dx 
= 
\ainner
\int\limits_0^\infty f_2(x) dx
\geq
\ainner
\max
\left\{
0,
\frac{1}{2}
- 
\frac{2}{\sqrt{\pi}} \frac{\sqrt{\frac{p}{1-p}}\ainner}{\sigmainner}
\right\}
\end{align*}

and 
\begin{equation*}
\begin{split}
&\ainner
\int\limits_0^\infty
x^2
(p f_1(x) + (1-p) f_2(x))
dx 
\leq
\ainner
\biggl(
p
\int\limits_0^{
\hspace{-0.1cm}
\mu_1 \ainner
}
x^2
f_1(x)
dx
+
p
\int\limits_0^\infty
\hspace{-0.1cm}
\hspace{-0.1cm}
\left(
x + \mu_1 \ainner
\right)^2
f_{\inner{\bx, \be_2}}\left(\frac{x}{\sigmainner}\right)
dx
+ 
\frac{(1-p) (1- \ainner^2)}{2}
\biggr)
\\
&= 
\ainner
\biggl(
p
\hspace{-0.1cm}
\int\limits_0^{
\hspace{-0.1cm}
\mu_1 \ainner
}
x^2
f_1(x)
dx
+
p
\hspace{-0.1cm}
\int\limits_0^\infty
\hspace{-0.1cm}
x^2
f_{\inner{\bx, \be_2}}\left(\frac{x}{\sigmainner}\right)
dx
+
p\hspace{-0.1cm}\int\limits_0^\infty
\hspace{-0.1cm}
2x\mu_1\ainner
f_{\inner{\bx, \be_2}}\left(\frac{x}{\sigmainner}\right)
dx+p
\left(\mu_1 \ainner\right)^2
\int\limits_0^\infty
f_{\inner{\bx, \be_2}}\left(\frac{x}{\sigmainner}\right)
dx+\frac{(1-p) (1- \ainner^2)}{2}
\biggr)
\\
&\leq
\ainner
\Biggl[
p\biggl(
\frac{2}{3} \left(\mu_1 \ainner\right)^3
+
\frac{ (1 - \ainner^2)}{2}
+
2 \sqrt{\frac{1}{2 \pi}} \mu_1 \ainner
+
\left(\mu_1 \ainner\right)^2
\frac{1}{2}
\biggr)
+ 
\frac{(1-p) (1- \ainner^2)}{2}
\Biggr]
\leq 
\frac{\ainner \sigmainner^2}{2} + \bigTheta{\sqrt{p}\ainner^2}
\end{split}
\end{equation*}

Thus in Case 1 have.
\begin{align*}
\expectation [\inner{g_\bu(\bx), \bu^*}] 
\geq
\bigTheta{\frac{\ainner^2}{\sqrt{p}}}
+ 
\frac{\ainner}{\sigmainner}
\biggl( \frac{1}{2} - 2 \frac{\sqrt{\frac{p}{1-p}}\ainner}{\sigmainner}
\biggr)
- 
\frac{\ainner \sigmainner^2}{2} - \bigTheta{\sqrt{p}\ainner^2}
\geq 
\bigTheta{\frac{\ainner^2}{\sqrt{p}}}.
\end{align*}
Analogously in Case 2: 
$\expectation [\inner{g_\bu(\bx), \bu^*}] \geq \bigTheta{\ainner}$
\egroup
\end{proof} 
\begin{lemma}
    \label{relu2ScoreConcentration}
    For $n = \Omega(d)$ have
    $$
    \probability\left[
    \max_{\hat{\bu} \in \bbS_{d-1}} \myabs{\sumin \frac{\psi(\inner{\bX_i, \hat{\bu}})}{n} 
    -
    \expectation[\psi(\inner{\bX, \hat{\bu}})]
    }
    \geq \delta
    \right]
    \leq
    \exp \left( - \bigTheta{\frac{n \delta^2}{d \log(1/\delta)}} \right)
    $$
\end{lemma}

\begin{proof}
First make the distinction into $\bY_j \sim 
\Data_{p,j}^{n_j}$ for $j\in \{1,2\}$, where $\sum_{j\in\{1,2\}} n_j = n$
We will first note that $\bA_j(\bu) = \psi(\bY_j \bu) - \expectation[\psi(\bY_j \bu)]$ is subexponential and thus 
$
    \oneOrclizNorm{
    \sum\limits_{j\in\{1,2\}}
    \sum\limits_{i = 1}^{n_j}
    \frac{\bA_{j,i}(\bu)}{n_j}} \leq \bigTheta{\frac{1}{n}}
$.
Thus, we can obtain
$$
\probability\left[ 
    \myabs{
    \sum\limits_{j\in\{1,2\}}
    \sum\limits^{n_j}_{i = 1} 
    \frac{\bA_{j,i}(\bu)}{n}} 
    \geq \delta_c
\right]
\leq
\exp
\left(
- \bigTheta{
\delta_c^2 n
}
\right)
$$
Let $\epsilonNet$ be the minimum size $\epsilon$-Net of the $d$-dimensional unit sphere.
Thus we know
$\myabs{\epsilonNet} \leq \left(\frac{3}{\epsilon} \right)^d$.
$$
\probability\left[ 
    \max_{\bu \in \epsilonNet}
    \myabs{
    \sum\limits_{j\in\{1,2\}}
    \sum\limits^{n_j}_{i = 1}
    \frac{\bA_{j,i}(\bu)}{n}} 
    \geq \delta_c
\right]
\leq
2 \exp
\left(
d \log\left(\frac{3}{\epsilon}\right)
- 
\bigTheta{\delta_c^2 n}
\right)
$$
Next we bound the maximum deviation for $\bu$ for $\mynorm{\bu' - \bu} < \epsilon$
\begin{flalign*}
    \myabs{
        \sumin \frac{\psi(\inner{\bX_i, \bu})}{n}
        -
        \sumin \frac{\psi(\inner{\bX_i, \bu'})}{n}
    }
    &
    \leq 
    \left(
        \sqrt{\sumin\frac{\psi(\inner{\bX_i, \bu})}{n}}
        +
        \sqrt{\sumin \frac{\psi(\inner{\bX_i, \bu'})}{n}}
    \right)
    \myabs{
        \sqrt{\sumin \frac{\psi(\inner{\bX_i, \bu})}{n}}
        -
        \sqrt{\sumin \frac{\psi(\inner{\bX_i, \bu'})}{n}}
    }
    \\
    &
    \leq 
    \frac{2 \opnorm{\bX} \mynorm{\bX (\bu - \bu')}}{n}
    \leq 
    \frac{2 \epsilon \|\bX\|_{op}^2}{n}.
\end{flalign*}
We continue by bounding for all $\bu \in \bbS_{d-1}$ by applying the triangle inequality
\begin{flalign*}
    \myabs{
    \sum\limits_{j\in\{1,2\}}
    \sum\limits^{n_j}_{i = 1}
    \frac{\bA_{j,i}(\bu)}{n}}
    &
    \leq
    \max_{\bu' \in \epsilonNet}
    \myabs{
    \sum\limits_{j\in\{1,2\}}
    \sum\limits^{n_j}_{i = 1}
    \frac{\bA_{j,i}(\bu)}{n}}
    +
    \frac{2 \epsilon \|\bX\|_{op}^2}{n}.
\end{flalign*}

Finally we notice that $\bX$ can be decomposed into the union of $\bY_1, \bY_2$ with $n_1 + n_2 = n$, where $n_1 \sim \mathop{\mathrm{Binomial}}(n, p)$.
Thus, by the Chernoff bound for the Binomial distribution, we have
$$
\probability\left[\myabs{
    \expectation[\psi(\inner{\bX, \bu})] - \sum\limits_{j\in\{1,2\}} 
    \frac{n_j}{n} \expectation\left[\psi\left(\inner{\bY_j, \bu}\right)\right]
}
\geq \delta_p
\right]
\leq 2 \exp\left( - \bigTheta{\delta_p^2 n p} \right)
$$
Thus, we can prove the Theorem by choosing $\epsilon = \bigTheta{\delta} $ for $n > d > 1$ and $\delta < 1$.
The second to last step follows bounding $\opnorm{\bX} \leq \opnorm{\bX (\bu^* {\bu^*}^\top)} + \opnorm{\bX (\bI - \bu^* {\bu^*}^\top)} \leq \bigTheta{\sqrt{n} + t}$ with probability $1 - 2 \exp(- t^2)$~\cite{vershynin2018high}.
\begin{flalign*}
    &\probability\left[
    \max_{\bu \in \bbS_{d-1}} \myabs{\sumin \frac{\psi(\bX_i \bu)}{n} -\expectation[\psi(\bX \bu)]} \geq \delta
    \right]
\\
    &\leq
    \probability\left[ 
    \max_{\bu \in \epsilonNet}
    \myabs{
    \sum\limits_{j\in\{1,2\}}
    \sum\limits^{n_j}_{i = 1}
    \frac{\bA_{j,i}}{n}} 
    \geq \frac{\delta}{3}
    \right]
    +
    \probability\left[
    \frac{2 \epsilon \opnorm{\bX}}{n}
    \geq \frac{\delta}{3}
    \right]
    +
    \probability\left[\myabs{
    \expectation[\psi(\inner{\bX, \bu})] - \sum\limits_{j\in\{1,2\}} \expectation\left[\psi\left(\inner{\bY_j, \bu}\right)\right]}
    \geq \frac{\delta}{3}
    \right]
\\
    &\leq
    \exp
    \left(
    d \log\left(\frac{3}{\epsilon}\right) - 
    \bigTheta{\delta^2 n}
    \right)
    + 
    2 \exp(- n)
    +
    \exp\left( - \bigTheta{\delta^2 n p} \right)
\leq
    \exp \left( - \bigTheta{\frac{n \delta^2}{d \log(1/\delta)}}\right).
\end{flalign*}
\end{proof} 
\begin{lemma}[Testing]
    \label{relu2Testing}
    For any $\ainner > \ainnerdash \geq a_{min} > 0$ there exists $\bar{p} > 0$ such that there exists a threshold $t$ such that with probability at least $1 - \exp \left( - \bigTheta{\frac{n \Delta^2}{d \log(1/\Delta)}} \right)$ where $\Delta = \ainner - \ainnerdash$ we have
    $$
    \sumin \frac{\psi(\inner{\bX_i, \bu})}{n} \geq t
    \,
    \text{and}
    \,
    \sumin \frac{\psi(\inner{\bX_i, \bu'})}{n} \leq t.
    $$
\end{lemma}
\begin{proof}
We begin by bounding the expectation of the score function.
\begin{align*}
    \expectation[\phi(\inner{\bX, \bu})]
    =
    & 
    p
    \int_0^\infty
    \left(
    x^2
    f_{\mathcal{N}(0,1)}\left(
        \frac{
            x - \ainner \sqrt{(1-p)/p}
        }{
            \sqrt{1 - \ainner^2}
        }
    \right)
    dx
    \right)
    +
    (1-p)
    \int_0^\infty
    \left(
    x^2
    f_{\mathcal{N}(0,1)}\left(
        \frac{
            x - \ainner \sqrt{p/(1-p)}
        }{
            \sqrt{1 - \ainner^2}
        }
    \right)
    dx
    \right)
    \\
    =
    &
    (1-p)   \ainner^2
    +
    p (1 - \ainner^2)
    - 
    p \int_{-\infty}^0\left(
    x^2
    f_{\mathcal{N}(0,1)}\left(
        \frac{
            x - \ainner \sqrt{(1-p)/p}
        }{
            \sqrt{1 - \ainner^2}
        }
    \right)
    dx
    \right)
    \\
    &
    +
(1-p)
\int_0^{\infty}
    \left(
    x^2
    f_{\mathcal{N}(0,1)}\left(
        \frac{
            \ainner \sqrt{p/(1-p)}
        }{
            \sqrt{1 - \ainner^2}
        }
    \right)
    dx
    \right).
\end{align*}

Thus for each $c_1 > 0$ there exists a $\bar{p} > 0$ such that for all $p \leq \bar{p}$ and $\ainner \geq l$
\begin{align*}
    \ainner^2 (1-p) + 
    (1-\ainner^2)
    \left(
    \frac{1+p}{2}
    - c_{1}
    \right)
    \leq
    \expectation[\phi(\inner{\bX, \bu})]
    \leq
    \ainner^2 (1-p) + (1-\ainner^2) \frac{1+p}{2}
\end{align*}
and for $\ainner \leq 0$ have $\expectation[\phi(\inner{\bX, \hat{\bu}})] \leq \frac{1}{2}$.

Thus, we can bound the difference in expectation for $\bu$ and $\bu'$
\begin{align*}
    \expectation[\phi(\inner{\bX, \bu})] - \expectation[\phi(\inner{\bX, \bu'})]
    \geq
    (\ainner^2 - \ainnerdash^2) \frac{1-3p}{2} - c_1 (1 - \ainner^2).
\end{align*}
Thus for $\bu$ and $\bu'$ there exists a $\bar{p} > 0$ such that $\expectation[\phi(\inner{\bX, \bu})] - \expectation[\phi(\inner{\bX, \bu'})] > 0$
By \autoref{relu2ScoreConcentration} we know that 
$\myabs{\sumin \frac{\phi(\inner{\bX_i, \bu})}{n} -\expectation[\phi(\inner{\bX, \bu})]} \leq \delta$ with probability at least $1 - \exp \left( - \bigTheta{\frac{n \delta^2}{d \log(1/\delta)}} \right)$.
The lemma follows by choosing $\delta = \frac{ \myabs{ \expectation[\phi(\inner{\bX, \bu})] - \expectation[\phi(\inner{\bX, \bu'})]}}{2}$ and $t$ accordingly.
\end{proof}
\subsection{Proof of \autoref{relu2_sc}}
As previously discussed, we will split the proof of \autoref{relu2_sc} into parts for each execution of \autoref{the_alg}.
\begin{proof}
Combining \autoref{relu2Expectation} and \autoref{relu2Concentration} and choosing $\delta = \bigTheta{\min \left\{ \frac{\ainner^2}{\sqrt{p}} , \ainner\right\}}$ we obtain 
$$
    \sumin \inner{\frac{g_{\bu}(\bX_i)}{n}, \bu^*}
    \geq 
    \bigTheta{
        \min \left\{ \frac{\ainner^2}{\sqrt{p}} , \ainner\right\}
    }
$$
with probability at least $1 - \bigTheta{\frac{1}{s}}$ if $\ainner \geq a = \bigTheta{\frac{1}{\sqrt{d p}}}$ and $n = \bigOmegas{d^2 p^2}$.
Next, notice that with probability at least $1 - \bigTheta{\frac{1}{s}}$.
$$
    \mynorm{\frac{\phi'(\bX \bu)}{n}} 
    \leq 
    \mynorm{\frac{\bX \bu}{n}} 
    \leq
    \opnorm{\frac{(\bbI_d - \bu^* {\bu^*}^\top) \bX}{n}}
    +
    \mynorm{\bX \bu^*}
    \leq
    \bigTheta{
        \frac{1}{\sqrt{n}}
    }.
$$
By utilizing that $(\bbI_d - \bu^* {\bu^*}^\top) \bX$ is a gaussian random matrix and that $\mynorm{\bX \bu^*}^2$ follows a binomial(scaled) distribution.
Thus, we can bound using \autoref{orthogonalGradientNorm}
$$
    \mynorm{\frac{\phi'(\bX \bu)^\top \bX \correction}{n}} 
    \leq
    \bigTheta{
        \sqrt{\frac{d}{n}}
    }.
$$
Thus by applying \autoref{gradientNormRewrite} we obtain  probability $1 - \bigTheta{\frac{1}{s}}$ 
\begin{align}
    \mynorm{\frac{\sumin g_\bu(\bX_i)}{n}} 
    &\leq
    \sqrt{
        \inner{ \frac{\sumin g_\bu(\bX_i)}{n}, \bu^* }^2 \left(1 + \frac{\ainner^2}{1 - \ainner^2}\right)
        + \bigTheta{\frac{d}{n}}
    }.
    \label{normBound}
\end{align}

\medskip
First Execution of \autoref{the_alg}:
There exist parameters $t_1$ and $\eta_1 = \Omega(\sqrt{d} p)$ such that for $a_1 = \bigTheta{\frac{1}{\sqrt{p d}}}$ there exists a constant $b_1 \in (0, 1)$ such that
the first execution of \autoref{the_alg} fulfills the criteria of \autoref{mainTheorem} for $n = \widetilde{\cO}(d^2 p^2)$.
\medskip

By \autoref{relu2Init} we have $\ainner \geq \bigTheta{\frac{1}{\sqrt{p d}}}$ and thus fulfilling \autoref{assInitialization}.

Using \eqref{normBound} we can bound
\begin{align}
\mynorm{\frac{\sumin g_\bu(\bX_i)}{n}} 
& \leq
\max \left\{
    2
    \sqrt{\inner{ \frac{\sumin g_\bu(\bX_i)}{n}, \bu^* }^2 \left(1 + \frac{\ainner^2}{1 - \ainner^2}\right)}
,
    \bigTheta{\sqrt{\frac{d}{n}}}
\right\}.
\label{toVerify}
\end{align}
Thus for all $\ainner \in \left(a_1, b_1\right)$ and $n = \bigOmegas{d^2 p^2}$ we have
\begin{align*}
    \ainner^{-1} \sumin \inner{\frac{g_{\bu}(\bX_i)}{n}, \bu^*}
    \geq 
    \mynorm{\frac{\sumin g_\bu(\bX_i)}{n}}.
\end{align*}
This can be verified by using the bound \eqref{toVerify}.
By applying \autoref{renormalization_helper1}, we can verify that \autoref{assGradient} is fulfilled.
Finally, using \autoref{relu2Testing} shows that \autoref{assTestability} is fulfilled.

\medskip
Second Execution of \autoref{the_alg}:
There exist parameters $t_2$ and $\eta_2$ such that for $a_2 = b_1 - \delta > 0$ and $b_2 = 1 - \beta$ for some $\epsilon > 0$ and $\delta > 0$
the second execution of \autoref{the_alg} fulfills the criteria of \autoref{mainTheorem} for $n = \widetilde{\cO}(d^2 p^2)$.
\medskip

As a result of the first execution of \autoref{the_alg}, there exists a $\bu$ such that $\ainner \geq b_1 - \delta$ and thus \autoref{assInitialization} is fulfilled.
For all $\bu$ for which $\ainner \in (a,b)$ 
there exists a constant choice for $\eta_2 > 0$ such that $\mynorm{\eta_2 \frac{\sumin g_\bu(\bX_i)}{n}}^2 \leq \eta_2 \ainner^{-1} \inner{ \frac{\sumin g_\bu(\bX_i)}{n}, \bu^*}$.
By applying \autoref{renormalization_helper2}, we can show \autoref{assGradient} is fulfilled.
Finally, \autoref{assTestability} can be fulfilled by \autoref{relu2Testing}. 
\end{proof}
\section{Proofs in \autoref{KurtosisSection}}
In this section use $\phi(x) = x^4$ and $\psi(x) = -|x|$.

\label{kurtosisProofs}
\begin{lemma}
\label{kurtosisExpectation}
Let $p < \frac{1}{3}$.
Then
$$
\expectation [\inner{g_{\bu}(\bX), \bu^*}] \geq \bigOmega{\frac{\ainner^3}{p}}.
$$
\end{lemma}
\begin{proof}
Define $\mu_0 = 0$ and $\mu_i = \frac{i}{\sqrt{p}}$ and $p_0 = 1-p$ and $p_i = \frac{p}{2}$ for $i \in \{1,-1\}$.
Thus
\begin{align*}
\expectation [\inner{g_{\bu}(\bX), \bu^*}] & =
    4
    \sum_{i \in \{ -1, 0, 1\}}
    p_i
    \left(
    3 \ainner (1 - \ainner^2)^2 (1 - \mu_i^2) 
    + \mu_i^2 \ainner^3 (1 - \ainner^2)(\mu_i^2 - 3)
    \right)
\\
& =
    4
    (1 - \ainner^2) \ainner^3 \left( \frac{1}{p} - 3 \right)
    \geq \bigOmega{\frac{\ainner^3}{p}}.
\end{align*}
\end{proof} \begin{lemma}
    \label{kurtosisConcentration}
    For $\delta \in (0, 1)$ have
    \begin{flalign*}
    \mathbb{P}
    \left[
    \myabs{
    \inner{\sumin \frac{g_\bu(\bX_i)}{n}, \bu^*}
    - 
    \inner{\expectation\left[g_\bu(\bX)\right], \bu^*}
    }
    \geq \delta
    \right]
    \leq
    \exp
    \left(
    -
    \bigTheta{
        \frac{
            n \delta^2
        }
        {
            \log(n \log(s))^2
            \left( 1 + \max\left\{\ainner^4 p^{-1}, \ainner^8 p^{-2}\right\} \right)
        }
    }
    \right).
    \end{flalign*}
\end{lemma}
\begin{proof}
Choose an ortho-normal basis $\bE = (\bu^*, \be_2, ... , \be_d) \in \mathbb{R}^{d \times d}$ such that $\bu = \ainner \bu^* + \sqrt{1 - \ainner^2} \be_2$.
Next we will condition on $\hat{\nu}_i = \inner{\bX_i, \bu^*}$ for all $i$ where $\hat{\nu} \in \{-\sqrt{1/p}, 0, \sqrt{1/p}\}$.
For convenience we define $\hat{p} = \sumin \frac{\bbI_{\hat{\nu}_i \neq 0}}{n}$ as well as $n_j = \sumin \bbI_{\hat{\nu}_i = j \sqrt{1/p}}$ with $j \in \{ -1, 0, 1 \}$.
Additionally define $\bY_j = \left( \bX \bigl| \inner{\bX, \bu^*} = j \sqrt{\frac{1}{p}} \right)$.

Next observe that $\max_{i \in [n]}\myabs{\inner{\bX_i, \be_2}} \leq \bigTheta{\log(n \log(s))}$ with probability at least $1 - \bigTheta{\frac{1}{s}}$ by applying the union bound.
Thus for all $j$ we have with probability at least $1 - \bigTheta{\frac{1}{s}}$
$$
    \myabs{\expectation \left[\inner{g_{\bu}(\bY_j), \bu^*}\right] - \inner{g_{\bu}(\bY_j), \bu^*}}
    \leq 
    \bigTheta{\max\left\{ 1, \ainner^2 \left(j \frac{1}{\sqrt{d}}\right)^2, \ainner^4 \left(j \frac{1}{\sqrt{d}}\right)^3 \right\} \log(n \log(s))}
$$
by 
By Hoeffdings inequality we have
\begin{align*}
\bbP 
\left(
    \myabs{
    \sumin \frac{ \inner{g_\bu(\bX_{i}), \bu^*}}{n}
    -
    \sum_{j \in \{-1, 0, 1\}}
    \frac{n_j \bbE\left[ \inner{g_\bu(\bY_j), \bu^*} \right]}{n}
    }
    \geq \delta
\right)
\leq
\exp
\left(
-
\bigTheta{
    \frac{\delta^2}
    {
    \log(n \log(s))^2
    \sumin
    \frac{\max \left\{ 1, \ainner^4 \nu_i^4, \ainner^8 \nu_i^6 \right\}}{n^2}
    }
}
\right)
\\
\leq
\exp
\left(
-
\bigTheta{
    \frac{n \delta^2}
    {
    \log(n \log(s))^2
    \left( 1 + \hat{p} \max\left\{\ainner^4 p^{-2}, \ainner^8 p^{-3}\right\} \right)
    }
}
\right).
\end{align*}

By applying the Chernoff bound for the Binomial distribution, we obtain
$$
\bbP
\left[
\myabs{
    \sum_{j \in \{-1, 0, 1\}}
    \frac{n_j \bbE\left[ \inner{g_\bu(\bY_j), \bu^*} \right]}{n}
    -
    \expectation \left[ \sumin \inner{g_\bu(\bX_i), \bu^*} \right]
}
\geq 
\delta
\right]
\leq 
\exp\left( - \bigTheta{\frac{n p \delta^2}{\max\left\{ \ainner, \frac{\ainner^3}{p} \right\}^2}} \right).
$$

Finally, we can combine all bounds to obtain the full statement.
For any constant $\bar{\delta} \in (0, 1)$ have
\begin{flalign*}
&\bbP 
\left(
    \myabs{
    \sumin \frac{ \inner{g_\bu(\bX_i), \bu^*}}{n}
    -
    \bbE[\inner{g_\bu(\bX), \bu^*}]
    }
    \geq \delta
\right)
\leq
\bbP 
\left(
    \myabs{
    \sumin \frac{ \inner{g_\bu(\bX_{i}), \bu^*}}{n}
    -
    \sum_{j \in \{-1, 0, 1\}}
    \frac{n_j \bbE\left[ \inner{g_\bu(\bY_j), \bu^*} \right]}{n}
    }
    \geq \frac{\delta}{2}
    \Biggl|
    \myabs{\hat{p} - p} \leq p \bar{\delta}
\right)
\\ &+
\bbP 
\left(
    \myabs{\hat{p} - p} \geq p \bar{\delta}
\right)
+
\bbP
\left[
\myabs{
    \sum_{j \in \{-1, 0, 1\}}
    \frac{n_j \bbE\left[ \inner{g_\bu(\bY_j), \bu^*} \right]}{n}
    -
    \expectation \left[ \inner{g_\bu(\bX), \bu^*} \right]
}
\geq 
\frac{\delta}{2}
\right]
\\ &\leq
\exp
\left(
-
\bigTheta{
    \frac{
        n \delta^2
    }
    {
        \log(n \log(s))^2
        \left( 1 + p \max\left\{\ainner^4 p^{-2}, \ainner^8 p^{-3}\right\} \right)
    }
}
\right).
\end{flalign*}
\end{proof} 
\begin{lemma}
    \label{absScoreConcentration}
    For $n = \Omega(d)$ have
    $$
    \probability\left[
    \max_{\hat{\bu} \in \bbS_{d-1}} \myabs{\sumin \frac{\psi(\inner{\bX_i, \hat{\bu}})}{n} 
    -
    \expectation[\psi(\inner{\bX, \hat{\bu}})]
    }
    \geq \delta
    \right]
    \leq
    \exp \left( - \bigTheta{\frac{n \delta^2}{d}} \right).
    $$
\end{lemma}

\begin{proof}
First, make the distinction into $\bY_j \sim 
\cN\left(\frac{j \bu^*}{\sqrt{p}}, \bI_d - \bu^* {\bu^*}^\top \right)^{n_j}$ for $j \in \{-1, 0, 1\}$, where $\sum_{j\in\{-1, 0, 1\}} n_j = n$.
We will first note that $\bA_j(\bu) = \psi(\bY_j \bu) - \expectation[\psi(\bY_j \bu)]$ is subgaussian and thus
$$
\probability\left[ 
    \myabs{
    \sum\limits_{j\in\{-1, 0, 1\}}
    \sum\limits^{n_j}_{i = 1} 
    \frac{\bA_{j,i}(\bu)}{n}} 
    \geq \delta_c
\right]
\leq
\exp
\left(
- \bigTheta{\delta_c^2 n}
\right).
$$
Let $\epsilonNet$ be the minimum size $\epsilon$-Net of the $d$-dimensional unit sphere.
Thus we know
$\myabs{\epsilonNet} \leq \left(\frac{3}{\epsilon} \right)^d$.
$$
\probability\left[ 
    \max_{\bu \in \epsilonNet}
    \myabs{
    \sum\limits_{j\in\{-1, 0, 1\}}
    \sum\limits^{n_j}_{i = 1}
    \frac{\bA_{j,i}(\bu)}{n}} 
    \geq \delta_c
\right]
\leq
\exp
\left(
d \log\left(\frac{3}{\epsilon}\right)
- 
\bigTheta{\delta_c^2 n}
\right).
$$

First we bound the maximum deviation for $\bu'$ for $\mynorm{\bu' - \bu} < \epsilon$
\begin{flalign*}
    \myabs{
        \sumin \frac{\psi(\inner{\bX_i, \bu})}{n}
        -
        \sumin \frac{\psi(\inner{\bX_i, \bu'})}{n}
    }
    &
    \leq \frac{\myonenorm{\bX (\bu - \bu')}}{n}
    \leq \frac{\mynorm{\bX (\bu - \bu')}}{\sqrt{n}}
    \leq \frac{\epsilon \opnorm{X}}{\sqrt{n}}.
\end{flalign*}
We continue by bounding for all $\bu \in \bbS_{d-1}$ by applying the triangle inequality
\begin{flalign*}
    \myabs{
    \sum\limits_{j\in\{-1, 0, 1\}}
    \sum\limits^{n_j}_{i = 1}
    \frac{\bA_{j,i}(\bu)}{n}}
    &
    \leq
    \max_{\bu' \in \epsilonNet}
    \myabs{
    \sum\limits_{j\in\{-1, 0, 1\}}
    \sum\limits^{n_j}_{i = 1}
    \frac{\bA_{j,i}(\bu')}{n}} 
    +
    \frac{\epsilon \opnorm{X}}{\sqrt{n}}.
\end{flalign*}

Finally we notice that $\bX$ can be decomposed into the union of $\bY_{-1}, \bY_0, \bY_1$.
Thus, by applying the Chernoff bound for the Binomial Distribution we obtain
$$
\probability\left[\myabs{
    \expectation[\psi(\inner{\bX, \bu})] - \sum\limits_{j\in\{-1, 0, 1\}} \expectation\left[\psi\left(\inner{\bY_j, \bu}\right)\right]
}
\geq \frac{\delta_p}{\sqrt{p}}
\right]
\leq 2 \exp\left( - \bigTheta{\delta_p^2 n p} \right).
$$
Thus, we can prove the Theorem by choosing $\epsilon = \bigTheta{\delta}$ for $n > d > 1$ and $\delta < 1$.
The second to last step follows bounding $\opnorm{\bX} \leq \opnorm{\bX (\bu^* {\bu^*}^\top)} + \opnorm{\bX (\bI - \bu^* {\bu^*}^\top)} \leq \bigTheta{\sqrt{n} + t}$ with probability $1 - 2 \exp(- t^2)$~\cite{vershynin2018high}.
\begin{align*}
    &\probability\left[
    \max_{\bu \in \bbS_{d-1}} \myabs{\sumin \frac{\psi(\bX_i \bu)}{n} -\expectation[\psi(\bX \bu)]} \geq \delta
    \right]
\\
    & \leq 
    \probability\left[ 
    \max_{\bu \in \epsilonNet}
    \myabs{
    \sum\limits_{j\in\{-1, 0, 1\}}
    \sum\limits^{n_j}_{i = 1}
    \frac{\bA_{j,i}}{n}} 
    \geq \frac{\delta}{3}
    \right]
    +
    \probability\left[
    \frac{\epsilon \opnorm{X}}{\sqrt{n}}
    \geq \frac{\delta}{3}
    \right]
    +
    \probability\left[\myabs{
    \expectation[\psi(\inner{\bX, \bu})] - \sum\limits_{j\in\{-1, 0, 1\}} \expectation\left[\psi\left(\inner{\bY_j, \bu}\right)\right]}
    \geq \frac{\delta}{3}
    \right]
\\
     & \leq
    \exp
    \left(
    d \log\left(\frac{3}{\epsilon}\right) - 
    \bigTheta{\delta^2 n}
    \right)
    + 
    2 \exp(- n)
    +
    \exp\left( - \bigTheta{\delta^2 n p^2} \right)
\leq 
    2 \exp \left( - \bigTheta{\frac{n \delta^2}{d}} \right).
\end{align*}
\end{proof} \begin{lemma}
\label{absTesting}
For all $\ainner > \inner{\bu', \bu^*}$ and $\delta > 0$.
there exists a threshold $t$ such that 
$$
    \sumin \frac{\psi(\inner{\bX_i, \bu})}{n} \geq t
    \,
    \text{and}
    \,
    \sumin \frac{\psi(\inner{\bX_i, \bu'})}{n} \leq t
$$
with a probability of at least
$1 - 2 \exp \left( - \bigTheta{\frac{n \delta^2}{d}} \right)$.
\end{lemma}
\begin{proof}
By applying the expectation of the half-normal distribution, we obtain the following upper and lower bounds.
\begin{align*}
    -\sqrt{\frac{2}{\pi}}
    \sqrt{1 - \ainner^2} (1-p) 
    - \sqrt{p} \ainner^2
    \geq
    \expectation[\psi(\inner{\bX, \bu})] 
    \geq 
    -
    \sqrt{\frac{2}{\pi}}
    \sqrt{1 - \ainner^2} (1-p)
    - \sqrt{p}.
\end{align*}

Thus, if for all $\bu$
$$
    \myabs{
    \sumin \frac{\psi(\inner{\bX_i, \bu})}{n}
    - 
    \expectation[\psi(\inner{\bX, \bu})] 
    }
    \leq 
    \delta
$$
then 
$$
    -
    \sqrt{\frac{2}{\pi}}
    \sqrt{1 - \ainner^2} (1-p)
    - \sqrt{p}
    - \delta
    >
    -
    \sqrt{\frac{2}{\pi}}
    \sqrt{1 - \inner{\bu', \bu^*}^2} (1-p) 
    - \sqrt{p} \inner{\bu', \bu^*}^2
    + \delta
$$
indicates the existence of the threshold.
The Lemma follows by applying \autoref{absScoreConcentration} and reordering terms.
\end{proof} \begin{lemma}
\label{KurtosisGN}
With probability at least $1 -\bigTheta{\frac{1}{s}} - \exp\left(- \bigTheta{\frac{n p}{\log(n s)}}\right)$
have
$$
    \mynorm{\frac{\phi'(\bX \bu)}{n}} 
    \leq
    \bigTheta{
        \frac{1 + \frac{\ainner^3}{p}}{\sqrt{n}}
    }
$$
\end{lemma}
\begin{proof}
Choose an ortho-normal basis $\bE = (\bu^*, \be_2, ... , \be_d) \in \mathbb{R}^{d \times d}$ such that $\bu = \ainner \bu^* + \sqrt{1 - \ainner^2} \be_2$.
\begin{align*}
    \bbE[\phi'(\inner{\bX, \bu})^2]
    & = \bbE[\inner{\bX, \bu}^6]
    = \bbE\left[ \left(\ainner \inner{\bu^*, \bX} + \sqrt{1 - \ainner^2} \inner{\be_2 , \bX}\right)^6 \right]
    \\
    & \leq \bigTheta{1 + \bbE\left[\ainner \inner{\bu^*, \bX})^6\right]}
\leq \bigTheta{1 + \frac{\ainner^6}{p^2}}.
\end{align*}

For convenience we define $\hat{p} = \sumin \frac{\bbI_{\inner{\bX_i, \bu^*} \neq 0}}{n}$ as well as $n_j = \sumin \bbI_{\inner{\bX_i, \bu^*} = j \sqrt{1/p}}$ with $j \in \{ -1, 0, 1 \}$.
Additionally define $\bY_j = \left( \bX \bigl| \inner{\bX, \bu^*} = j \sqrt{\frac{1}{p}} \right)$.
For $\bar{\delta} \in (0,1)$
\begin{align}
& \bbP
\left[
    \myabs{
        \sumin \frac{\phi'(\inner{\bX_i, \bu})^2}{n}
        - 
        \bbE
        \left[
            \phi'(\inner{\bX, \bu})^2
        \right]
    }
    \geq
    \delta
\right] \nonumber
\\
& 
\leq \,
\bbP
\left[
    \myabs{
        \sumin \frac{\phi'(\inner{\bX_i, \bu})^2}{n}
        -
        \sum_{j \in \{-1, 0, 1\}}
        \frac{n_j}{n} \bbE \left[ \phi'\left(\inner{\bY_{j}, \bu}\right)^2 \right]
    }
    \geq
    \frac{\delta}{2}
    \,
    \bigg\vert
    \,
    \myabs{\hat{p} - p} \leq p \bar{\delta}
    \land
    \max_{i \in [n]} \inner{\bX_i, \be_2} \leq \bigTheta{\sqrt{\log(n s)}}
\right] \nonumber
\\ & 
+
\bbP 
\left[
    \myabs{\hat{p} - p} \geq p \bar{\delta}
\right]
+
\bbP 
\left[
    \max_{i \in [n]} \inner{\bX_i, \be_2} \geq \bigTheta{\sqrt{\log(n s)}}
\right]
+
\bbP
\left[
\myabs{
    \left(
    \sum_{j \in \{-1, 0, 1\}}
    \frac{n_j}{n} 
    \bbE \left[ \phi'\left(\inner{\bY_{j}, \bu}\right)^2 \right]
    \right)
    -
    \bbE \left[ \phi'(\inner{\bX, \bu})^2 \right]
}
\geq 
\frac{\delta}{2}
\right] \nonumber
\\
&
\leq \,
\exp\left( - \bigTheta{\frac{n \delta^2}{\log(n s) \left(1 + \frac{\ainner^6}{ p^2}\right)^2}}\right) +
2 \exp\left( - \frac{\bar{\delta}^2 n p}{3} \right)
+
2 \exp\left( \log(n) - \bigTheta{\log(ns)} \right)
+
\exp\left( - \bigTheta{\frac{\delta^2 n p}{\left(
1 + \frac{\ainner^6}{p^2}\right)^2}} \right) \label{concentration_ineqs1}
\\
&
\leq \,
\bigTheta{\frac{1}{s}}
+
\exp\left( - \bigTheta{\frac{\delta^2 n p}{\log(n s)\left(1+ \frac{\ainner^6}{p^2}\right)^2}} \right).\nonumber
\end{align}
Here \eqref{concentration_ineqs1} follows by applying Hoeffdings inequality and the Chernoff bound for the Binomial distribution.
The last term in \eqref{concentration_ineqs1} follows by bounding the difference $
    \bbE \left[ \phi'\left(\inner{\bY_{1}, \bu}\right)^2 \right] - \bbE \left[ \phi'\left(\inner{\bY_{0}, \bu}\right)^2 \right]
    =
    \bbE \left[ \phi'\left(\inner{\bY_{-1}, \bu}\right)^2 \right] - \bbE \left[ \phi'\left(\inner{\bY_{0}, \bu}\right)^2 \right] 
    \leq \bigTheta{1 + \frac{\ainner^6}{p^2}}
$ and also applying the Chernoff bound.Choosing $\delta = \bigTheta{1 + \frac{\ainner^6}{p^2}}$ yields 
$$
    \bbE[\phi'(\inner{\bX, \bu})^2] \leq \bigTheta{1 + \frac{\ainner^6}{p^2}}.
$$
With probability at least $1 - \bigTheta{\frac{1}{s}} - \exp\left(- \bigTheta{\frac{n p}{\log(n s)}}\right)$.
The lemma follows by computing the norm.
\end{proof} 
\subsection{Proof of \autoref{kurtosis_sc}}
\begin{proof}
As discussed, we will split the proof of \autoref{kurtosis_sc} into parts for each execution of \autoref{the_alg}.

\medskip
First Execution of \autoref{the_alg}:
    There exist parameters $t_1$ and $\eta_1 = \Omega(d p^2)$ such that for $a_1 = \bigTheta{\frac{1}{\sqrt{p d}}}$ and some constant $b_1 \in (0, 1)$
    the first execution of \autoref{the_alg} fulfills the criteria of \autoref{mainTheorem} for $n = \widetilde{\cO}(d^3 p^4)$.
\medskip

By the same argument as in \autoref{relu2Init}, we have $\ainner \geq \bigTheta{\frac{1}{\sqrt{p d}}}$ with sufficiently large probability, thus fulfilling \autoref{assInitialization}.

Using \autoref{kurtosisExpectation} and applying the concentration result of \autoref{kurtosisConcentration}, we obtain that
\begin{align}
    \sumin \inner{\frac{g_{\bu}(\bX_i)}{n}, \bu^*}
    =
    \bigOmega{
        \frac{\ainner^3}{p}
    }
    \label{kurtosisGradientFormula}
\end{align}
with probability of at least
$ 1 - \exp
\left(
-
\bigTheta{
    \frac{
        n (\ainner^3 p^{-1})^2
    }
    {
        \log(n \log(s))^2
        \left( 1 + \max\left\{\ainner^4 p^{-1}, \ainner^8 p^{-2}\right\} \right)
    }
}
\right)
$.

Next bound with probability at least 
$1 -\bigTheta{\frac{1}{s}} - \exp\left(- \bigTheta{\frac{n p}{\log(n s)}}\right)$
\begin{align}
\mynorm{\frac{\sumin g_\bu(\bX_i)}{n}} 
&\leq
\sqrt{
    \inner{ \frac{\sumin g_\bu(\bX_i)}{n}, \bu^* }^2 \left(1 + \frac{\ainner^2}{1 - \ainner^2}\right)
    + \mynorm{\frac{\phi'(\bX \bu)^\top \bX \correction}{n}}^2 
}
\label{grad_decomp}
\\
&\leq
\sqrt{
    \inner{ \frac{\sumin g_\bu(\bX_i)}{n}, \bu^* }^2 \left(1 + \frac{\ainner^2}{1 - \ainner^2}\right)
    + 
    \bigTheta{
        \left(1 + \frac{\ainner^6}{p^2}\right)\frac{d}{n}
    }
}
\label{ApplylemGradientBound}
\\
& \leq
\max \left\{
    2
    \sqrt{\inner{ \frac{\sumin g_\bu(\bX_i)}{n}, \bu^* }^2 \left(1 + \frac{\ainner^2}{1 - \ainner^2}\right)}
,
    2
    \sqrt{
        \bigTheta{\left(1 + \frac{\ainner^6}{p^2}\right)\frac{d}{n}}
    }
\right\}.
\label{c5}
\end{align}
Here, \eqref{grad_decomp} follows by by decomposing the norm.
\eqref{ApplylemGradientBound} follows by applying \autoref{KurtosisGN} and  \autoref{orthogonalGradientNorm}.
Finally using $\eta_1 = \bigOmega{p^2 d}$ we apply \autoref{renormalization_helper1} to demonstrate that \autoref{assGradient} is fulfilled if $n = \bigOmegas{d^3 p^4}$.

Finally, \autoref{absTesting} shows that \autoref{assTestability} is fulfilled.

\medskip
Second Execution of \autoref{the_alg}:
    There exist parameters $t_2$ and $\eta_2 = \Omega(d p^2)$ such that for $a_2 = b_1 - \delta$ and $b_2 = 1-\beta$ for some $3 > \epsilon > 0$
    the first execution of \autoref{the_alg} fulfills the criteria of \autoref{mainTheorem} for $n = \bigOmegas{d^3 p^4}$.
\medskip

\autoref{assInitialization} is fulfilled by the first execution of \autoref{the_alg}.
Using \eqref{c5} have $\mynorm{\frac{\sumin g_\bu(\bX_i)}{n}} = \bigO{\frac{1}{p}}$ if $\myabs{\ainner} \leq 1-\beta$.
Choosing $\eta_2 = \bigTheta{1} > 0$ yields that
${\eta_2 \ainner^{-1}\inner{\sumin \frac{g_\bu(\bX_i)}{n}, \bu^*} \geq \mynorm{\eta_2 \frac{\sumin g_\bu(\bX_i)}{n}}^2}$.
Thus applying \autoref{renormalization_helper2} we obtain that \autoref{assGradient} is satisfied if  $n = \bigOmegas{d^3 p^4}$.
Finally, \autoref{absTesting} shows that \autoref{assTestability} is fulfilled.

\end{proof}   
\section{Proofs in \autoref{ldplb}}
\subsection{Proof of \autoref{reduction}}
\begin{proof} 
First note we have $\bbE_{\cQ}[\relutwo{\bX}] = \bbE_{\cP}[\relutwo{\bX}$ if $\inner{\hat{\bu}, \bu^*} = 0$.
Thus, we can apply \autoref{relu2Testing} to obtain the result.
\end{proof} \autoref{akm} is nearly equlvalent to Lemma 6.7 in \citet{mao2022optimal}.
\begin{lemma}
\label{akm}
For $\alpha \in \mathbb{N}^n$, let $|\alpha| = \sum_{i=1}^n \alpha_i = \|\alpha\|_1$, and let $\|\alpha\|_0$ be the size of the support of $\alpha$. 
For $m \in [d]$, define a set 
\begin{equation}
\mathcal{A}(k, m) : = \left\{ \alpha \in \mathbb{N}^n : |\alpha| = k , \, \|\alpha\|_0 = m, \, \alpha_i \in \{0\} \cup \{ 3, 4, \dots \} \text{ for all } i \in [n] \right\} .
\end{equation}
Then we have 
$\left| \mathcal{A}(k, m) \right| \leq n^m k^k$.
\end{lemma}
\begin{proof}
$\left| \mathcal{A}(k, m) \right| \leq \binom{n}{k} k^k$.
\end{proof} \begin{lemma}
\label{herm_ic}
Have the Imbalanced Clusters RV $\bX \sim \Data_\nu(p)$
\begin{center}
    $\expectation[\hat{h}_0(\bX)] = 1$,
    $\expectation[\hat{h}_1(\bX)] = 0$,
    $\expectation[\hat{h}_2(\bX)] = 0$
\end{center}
And given $p < 0.5$ have for $k \geq 3$ 
$$|\expectation[\hat{h}_k(\bX)]| \leq k^{k/2} p^{1 - k/2}.$$
\end{lemma}
\begin{proof}
    For $p < 0.5$ have $\expectation[\bX^k] \leq 2 p^{1 - k/2}$.
    Thus, we can bound 
    $$|\expectation[h_k(\bX)]| = \frac{1}{\sqrt{k!}} \left( \sum\limits_{i=0}^k c_i \expectation[\bX^i]\right) \leq |\expectation[\bX^i]| \sqrt{k!}.$$ 
    With the last inequality following from $ \sum\limits_{i=0}^k |c_i| \leq k!$.
\end{proof} \begin{lemma}[\citet{mao2022optimal}]
\label{formula}
Consider the distribution $\cP$ in \autoref{testing} and suppose the first $D$ moments of $\nu$ are finite. For $\alpha \in \cN^N$, let $|\alpha| := \sum_{i=1}^N \alpha_i$. Then
\begin{equation}\label{eq:ld-formula}
\ldlr^2 = \sum_{d=0}^D \expectation[\langle \bu, \bu' \rangle^d] \sum_{\substack{\alpha \in \cN^N \\ |\alpha| = d}} \prod_{i=1}^N \left(\expectation_{x \sim \nu}[h_{\alpha_i}(x)]\right)^2
\end{equation}
where $\bu$ and $\bu'$ are drawn independently from $\mathcal{U}$.
\end{lemma} \begin{lemma}[\citet{mao2022optimal}]
\label{inner-product}
Let $\bu$ and $\bu'$ be independent uniform random vectors on the unit sphere in $\mathbb{R}^n$. For $k \in \mathbb{N}$, if $k$ is odd, then $\expectation[\langle \bu, \bu' \rangle^d] = 0$, and if $k$ is even, then 
$$ 
\expectation[\langle \bu, \bu' \rangle^k] \le (k/n)^{k/2}
$$
\end{lemma} \subsection{Proof of \autoref{thmLDPLB}}
The proof of \autoref{thmLDPLB} closely follows the proof of Theorem 4.5 in \citet{mao2022optimal}.

\begin{proof}
    We will use $\mathcal{A}(k, m)$ as defined in \autoref{akm} and note that 
    for $\alpha \in \mathcal{A}(k, m)$, we obtain that $\alpha \geq 3$ and thus $m \leq \lfloor k/3 \rfloor$ using \autoref{herm_ic}
    \begin{align*}
    \sum_{\substack{\alpha \in \Natural^n \\ |\alpha| = k}} \prod_{i=1}^n \left(\expectation[h_{\alpha_i}(x)]\right)^2 
    = \sum_{m=1}^{\lfloor k/3 \rfloor} \sum_{\substack{\alpha \in \mathcal{A}(k, m)}} \prod_{i=1}^n \left(\expectation[h_{\alpha_i}(x)]\right)^2
    \le \sum_{m=1}^{\lfloor k/3 \rfloor} \left| \mathcal{A}(k, m) \right| \prod_{i \in [n], \, \alpha_i \ne 0} \alpha_i^{2 \alpha_i} p^{2-\alpha_i}
    \le \sum_{m=1}^{\lfloor k/3 \rfloor} n^m k^{3k} p^{2 m - k}.
    \end{align*}
    By applying the closed form of the geometric series, we obtain
    \begin{align*}
    \sum_{\substack{\alpha \in \Natural^n \\ |\alpha| = k}} \prod_{i=1}^n \left(\expectation[h_{\alpha_i}(x)]\right)^2 
    \le k^{3k} n p^{2-k} \frac{(n p^2)^{\lfloor k/3 \rfloor} - 1}{n p^2 - 1}
    \le k^{3k} n p^{2-k} \frac{(n p^2)^{k/3}}{\frac{1}{2} n p^2}
    = 2 k^{3k} n^{k/3} p^{-k/3} .
    \end{align*}
    This combined with \autoref{inner-product} gives 
    \begin{equation*} 
    \expectation[\langle u,u' \rangle^k] \sum_{\substack{\alpha \in \Natural^n \\ |\alpha| = k}} \prod_{i=1}^n \left(\expectation[h_{\alpha_i}(x)]\right)^2 
    \le (k/d)^{k/2} \cdot 2 k^{3k} n^{k/3} p^{-k/3}
    = 2 \left( \frac{ k^{10.5} n }{ d^{3/2} p } \right)^{k/3} . 
    \end{equation*} 
    Finally, combining this with \autoref{formula}, we obtain  
    \begin{equation*} 
    \ldlr^2 = \sum_{k=0}^D \expectation[\langle u,u' \rangle^k] \sum_{\substack{\alpha \in \Natural^n \\ |\alpha| = k}} \prod_{i=1}^n \left(\expectation[h_{\alpha_i}(x)]\right)^2 
    \le 1 + 2 \sum_{k=3}^D \left( \frac{ k^{10.5} n }{ d^{3/2} p } \right)^{k/3} . 
    \end{equation*}
    This is true if $d^{1.5} p > n D^{c_{2}}$ for a sufficiently large constant $c_{2} > 0$ such that $\frac{k^{10.5} n }{ d^{3/2} p } < 1/4$. 
\end{proof}
\section{Additional Proofs}
\begin{lemma}
    \label{gradientNormRewrite}
    Choosing the ortho-normal basis $\bE = (\bu^*, \be_2, ... , \be_d) \in \mathbb{R}^{d \times d}$ such that $\bu = \ainner \bu^* + \sqrt{1 - \ainner^2} \be_2$.
    \begin{flalign*}
        \mynorm{\frac{\sumin g_\bu(\bX_i)}{n}} =
        \sqrt{
            \inner{ \frac{\sumin g_\bu(\bX_i)}{n}, \bu^* }^2
            \left(1 + \frac{\ainner^2}{1 - \ainner^2}\right)
            + \mynorm{\frac{\phi'(\bX \bu)^\top \bX \correction}{n}}^2
       }.
    \end{flalign*}
\end{lemma}
\begin{proof}
    $$
        \mynorm{\frac{\sumin g_\bu(\bX_i)}{n}} = 
        \sqrt{
            \inner{ \frac{\sumin g_\bu(\bX_i)}{n}, \bu^* }^2 + \inner{ \frac{\sumin g_\bu(\bX_i)}{n}, \be_2 }^2 + \sum_{i=3}^d \inner{ \frac{\sumin g_\bu(\bX_i)}{n}, \be_i }^2
        }
    $$
    By $\ainner \inner{ \frac{\sumin g_\bu(\bX_i)}{n}, \bu^* } + \sqrt{1-\ainner^2} \inner{ \frac{\sumin g_\bu(\bX_i)}{n}, \be_2 } = 0$ we can expand to
    \begin{flalign*}
        \mynorm{\frac{\sumin g_\bu(\bX_i)}{n} } & =
        \sqrt{
            \inner{ \frac{\sumin g_\bu(\bX_i)}{n}, \bu^* }^2
            \left(1 + \frac{\ainner^2}{1 - \ainner^2}\right)
            + \sum_{i=3}^d \inner{ \frac{\sumin g_\bu(\bX_i)}{n}, e_i }^2
       }.
    \end{flalign*}
\end{proof} \begin{lemma}
    \label{orthogonalGradientNorm}
    Choose a ortho-normal basis $\bE = (\bu^*, \be_2, ... , \be_d) \in \mathbb{R}^{d \times d}$ such that $\bu = \ainner \bu^* + \sqrt{1 - \ainner^2} \be_2$.
    Thus
    $$
        \mynorm{\frac{\phi'(\bX \bu)^\top \bX \correction}{n}} 
        =
        \bigO{
            \sqrt{d}
            \mynorm{\frac{\phi'(\bX \bu)}{n}}
        }.
    $$
    with a probability of at least 
    $1 - \bigO{\frac{1}{s}}$.
\end{lemma}

\begin{proof}
First notice that $\bX \correction \sim \cN(0,\correction)^n$ and thus that for some vector $\bv$
\begin{align*}
\bbP
\left[
    \mynorm{\sumin \frac{\bv_i}{\mynorm{\bv}} \X_i \correction}
    \geq 
    \sqrt{d} + t
\right]
&
\leq
2 \exp\left( - \bigTheta{t^2} \right).
\label{bound_norm2}
\end{align*}
The lemma follows by using $t = \sqrt{\log(s)}$ and choosing $\bv = \frac{\phi'(\bX \bu)}{n}$
\end{proof} \begin{lemma}
    \label{renormalization_helper1}
    For any $c_3 \in (0,1)$ such that 
    if 
        $
        \mynorm{\frac{\sumin g_\bu(\bX_i)}{n}} \leq 
        (1 - c_3)
        \frac{\inner{ \frac{\sumin g_\bu(\bX_i)}{n}, \bu^* }}{\ainner}$ 
    and 
        $\frac{\eta \inner{ \frac{\sumin g_\bu(\bX_i)}{n}, \bu^* }}{\ainner} \geq 1$
    we have
    \begin{align*}
        \frac{\langle \bu^*, \bu + \eta \inner{\sumin \frac{g_\bu(\bX_i)}{n}, \bu^*}}{\mynorm{\bu + \inner{\sumin \frac{g_\bu(\bX_i)}{n}, \bu^*}}}
        \geq
        \ainner \left(1 + \frac{c_6}{2}\right).
    \end{align*}
\end{lemma}

\begin{proof}
\begin{flalign*}  
\inner{ 
\bu^*, 
\frac{
    \bu + \eta \sumin \frac{g_\bu(\bX_i)}{n} 
}{
    \mynorm{\bu + \eta \frac{\sumin g_\bu(\bX_i)}{n}}
}
}
& \geq
\ainner
\frac{
    1 + \eta \ainner^{-1}\inner{\sumin \frac{g_\bu(\bX_i)}{n}, \bu^*}
}{
    \sqrt{1 + \eta^2 \mynorm{\frac{\sumin g_\bu(\bX_i)}{n}}^2}
}
\\ &
\geq
\ainner
\frac{
1 + \eta \ainner^{-1} \inner{\sumin \frac{ g_\bu(\bX_i)}{n}, \bu^*}
}{
1 + (1 - c_3) \eta \ainner^{-1} \inner{\sumin \frac{ g_\bu(\bX_i)}{n}, \bu^*}
}
\geq
\ainner \left(1 + \frac{c_3}{2}\right).
\end{flalign*}
\end{proof}
 \begin{lemma}
    \label{renormalization_helper2}
    If $\eta \ainner^{-1}\inner{\sumin \frac{g_\bu(\bX_i)}{n}, \bu^*} \geq \mynorm{\eta \frac{\sumin g_\bu(\bX_i)}{n}}^2$
    \begin{align*}
        \frac{
            \langle \bu^*, \bu + \eta \inner{\sumin \frac{g_\bu(\bX_i)}{n}, \bu^*}
        }{
            \mynorm{\bu + \eta \sumin \frac{g_\bu(\bX_i)}{n}}
        }
        \geq
        \ainner
        \min \left\{
            1 + \eta \inner{\sumin \frac{g_\bu(\bX_i)}{n}, \bu^*}
            , 2
        \right\}.
    \end{align*}
\end{lemma}

\begin{proof}
\begin{flalign*}  
\inner{ 
\bu^*, 
\frac{
    \bu + \eta \sumin \frac{g_\bu(\bX_i)}{n} 
}{
    \mynorm{\bu + \eta \frac{\sumin g_\bu(\bX_i)}{n}}
}
}
& \geq
\ainner
\frac{
    1 + \eta \ainner^{-1}\inner{\sumin \frac{g_\bu(\bX_i)}{n}, \bu^*}
}{
    \sqrt{
        1 + \mynorm{\eta \frac{\sumin g_\bu(\bX_i)}{n}}^2
    }
}
\\ &
\geq
\ainner
\sqrt{
    1 + \eta \ainner^{-1}\inner{\sumin \frac{g_\bu(\bX_i)}{n}, \bu^*}
}
\geq
\ainner 
\min
\left\{
    1 + \eta \inner{\sumin \frac{g_\bu(\bX_i)}{n}, \bu^*}
    , 2
\right\}.
\end{flalign*}
\end{proof} \begin{lemma}
    \label{lemProjectedGrad}
    For any $\phi(\cdot)$ and $\bx, \bu, \bu^* \in \mathbb{R}^d$ have
    $$
        \inner{ g_\bu(\bx), \bu^* }
        =
        \frac{\partial \phi(\inner{ \bx, \bu })}{\partial \inner{ \bx, \bu }} 
        (\inner{ \bx, \bu^* } - \inner{ \bx, \bu } \inner{ \bu, \bu^* }).
    $$
\end{lemma}

\begin{proof}
First recall the definition $g_\bu(\bx) = (\bI_d - \bu \bu^T) \frac{\partial \phi(\inner{\bu, \bx}}{\partial u}$.
Let us choose a new ortho-normal basis $\bE = (\bu^*, \be_2, ... , \be_d) \in \mathbb{R}^{d \times d}$.
By $\inner{ \bx, \bu } = \inner{ \bx, \bu^* } \inner{ \bu, \bu^* } + \left( \sum_{i = 2}^d \inner{ \bx, e_i } \inner{ \bu, e_i } \right)$

\begin{align*}
\inner{ g_\bu(\bx), \bu^* }
& = \inner{ (I - \bu\bu^T) \frac{\partial \phi(\inner{ \bx, \bu })}{\partial \bu}, \bu^* }
= \frac{\partial \phi(\inner{ \bx, \bu })}{\partial \inner{ \bx, \bu }} \left(
\left(1 - \inner{ \bu, \bu^* }^2\right) \inner{ \bx, \bu^* }
- \inner{ \bu, \bu^* } \left( \sum_{i = 2}^d \inner{ \bu, \be_i } \inner{ \bx, \be_i } \right)
\right)
\\ & =
\frac{\partial \phi(\inner{ \bx, \bu })}{\partial \inner{ \bx, \bu }} 
\left(
(1 - \inner{ \bu, \bu^* }^2) \inner{ \bx, \bu^* }
- \inner{ \bu, \bu^* }
(\inner{ \bx, \bu } - \inner{ \bu, \bu^* } \inner{ \bx, \bu^* })
\right)
=
\frac{\partial \phi(\inner{ \bx, \bu })}{\partial \inner{ \bx, \bu }} 
(\inner{ \bx, \bu^* } - \inner{ \bx, \bu } \inner{ \bu, \bu^* }).
\end{align*}

\end{proof}

\section{Supplement}
\label{supplement}
\subsection{Synthetic Benchmark with Additional Methods}
In \autoref{fig:spectral_methods_comparison}, we repeat the experiment of Fig. 2 with a lower dimension $d = 100$ to include the PPcovMCD method introduced in \citet{POKOJOVY2022107475}, which is more computationally intensive and would otherwise be intractable for a benchmark with $d = 300$.
Here, we can observe that PPcovMCD recovers the signal direction well in the setting with an Imbalanced Clusters planted vector.

\begin{figure}[H]
    \centering
    \begin{subfigure}[b]{0.49\textwidth}
        \includegraphics[scale = 0.5]{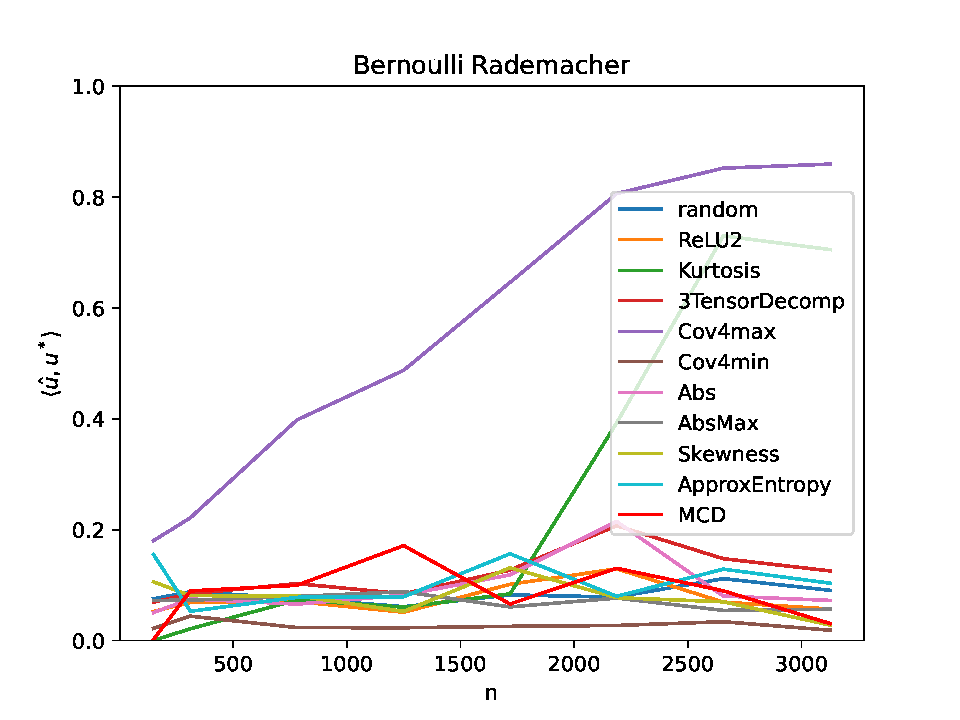}   
        \caption{Bernoulli-Rademacher}
        \label{br_comp}
    \end{subfigure}
    \hfill
    \begin{subfigure}[b]{0.49\textwidth}
        \includegraphics[scale = 0.5]{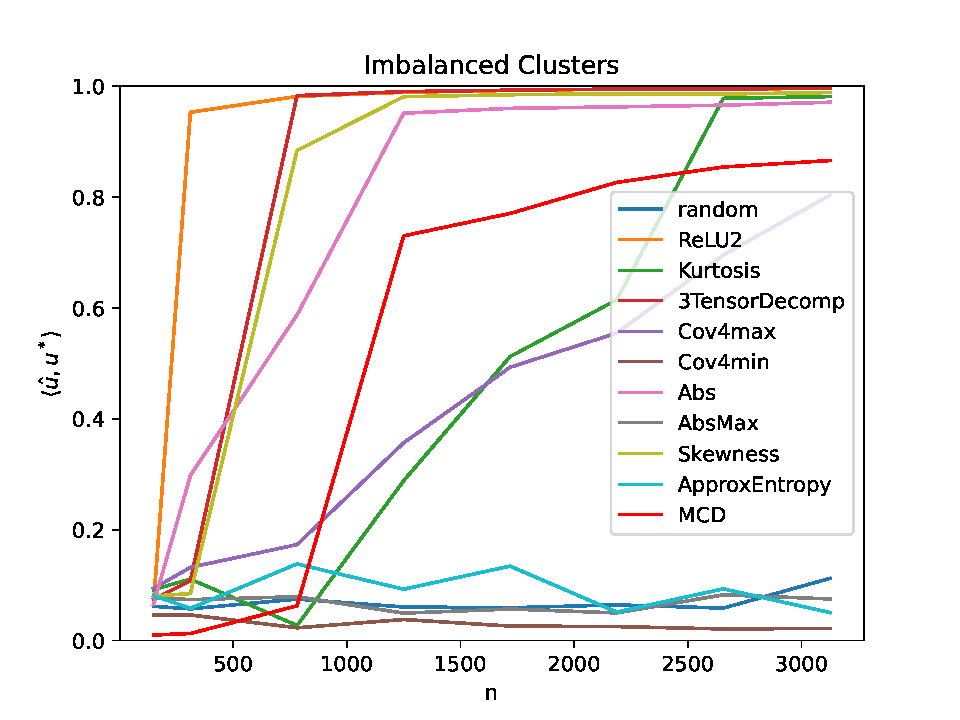}
        \caption{Imbalanced Clusters}
        \label{ic_comp}
    \end{subfigure}
    \caption{
        Comparison of different methods in the planted vector setting.
        We plot the average inner product between the signal direction and the recovered direction by each algorithm over 30 datasets.
    }
    \label{fig:spectral_methods_comparison}
\end{figure}

\subsection{Differing Cluster Variance}
A setting not covered by the planted vector setting is the recovery of a Gaussian mixture with differing variances.
We will follow the experiments described in \citet{Cabana2021}.
We sample from $(1-p) \Normal(0, \bbI_d) + p \Normal(\delta \ustar, \lambda \bbI_d)$ with $p = 0.3$ and $d = 100$.
We will use our gradient-based projection pursuit algorithm to recover $\ustar$ from the \emph{whitened} data.
The results are plotted in \autoref{cabana_experiment} where we show the average of $|\inner{\uhat, \ustar}|$ over 30 datasets.
Here, we can observe a distinct behavior with $\lambda = 1$, which resembles the Imbalanced Clusters planted vector as $\delta \to \infty$.
In contrast, for any sufficiently large $|\lambda - 1|$, samples belonging to the cluster with a lower variance will have a lower norm.
It can be observed that for our algorithm, recovery is easier if $\lambda > 1$.
We hypothesize that a specialized algorithm can exploit this fact to outperform our algorithm in sample complexity in this setting.

\begin{figure}
    \centering
        \begin{subfigure}[b]{0.4\textwidth}
            \centering
            \includegraphics[width=\textwidth]{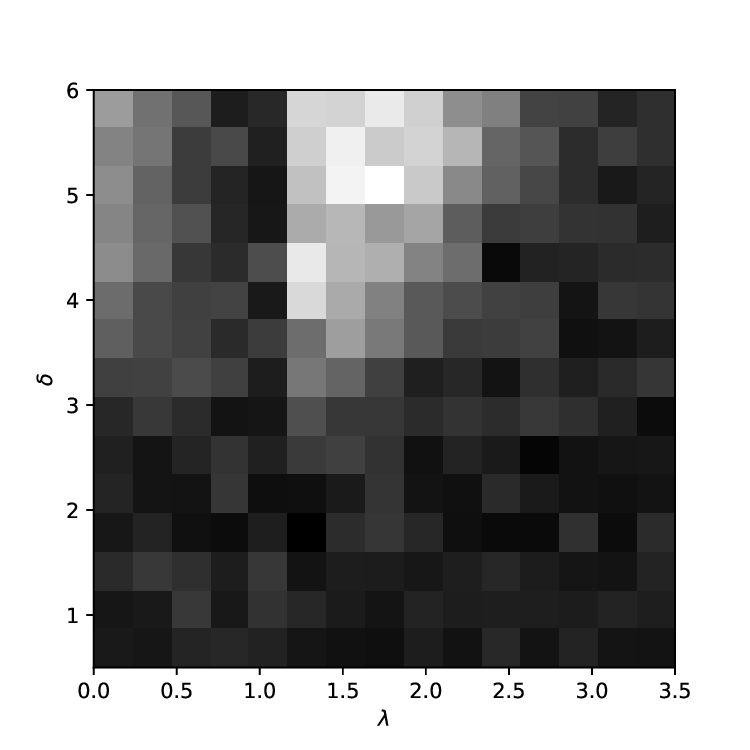}
            \caption{$n = 400$}
        \end{subfigure}
        \begin{subfigure}[b]{0.4\textwidth}
            \centering
            \includegraphics[width=\textwidth]{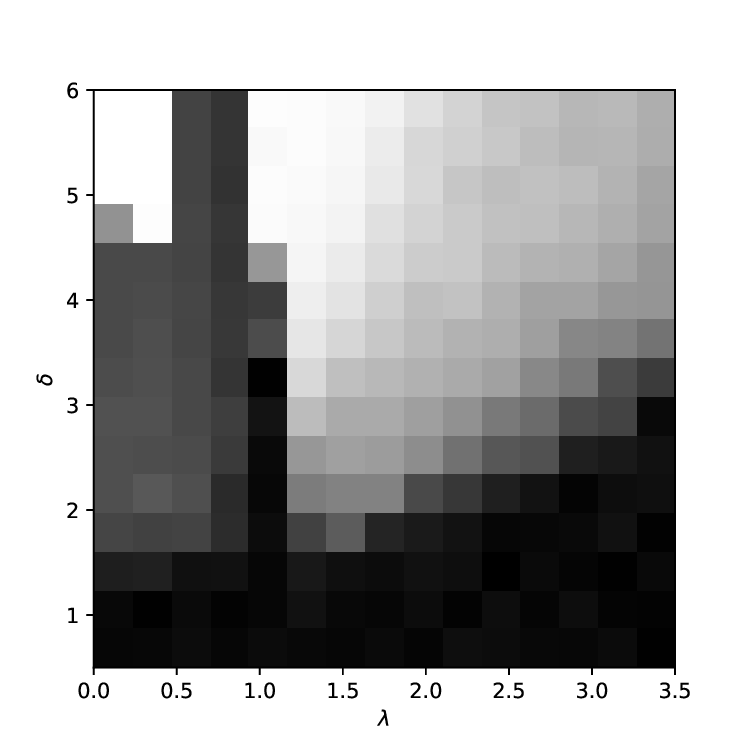}
            \caption{$n = 4000$}
        \end{subfigure}
    \caption{
        A plot of the average value of $|\inner{\uhat, \ustar}|$ of the predicted direction $\uhat$ and the signal direction $\ustar$.
        The experiment is repeated with $n = 400$ and $n = 4000$ samples.
        White pixels denote values close to 1, and black pixels denote values close to 0.
    }
    \label{cabana_experiment}
\end{figure}

\subsection{Anisotropic Gaussian Mixture}
We also demonstrate the efficacy of our approach by benchmarking on a Gaussian mixture where the Gaussians share a center, but one has an anisotropic variance.
For this, we choose a rank perturbation to the covariance matrix.
We obtain the distribution $(1-p) \Normal(0, \bbI_d) + p \Normal(0, \bbI_d + \ustar {\ustar}^\top (\lambda - 1))$ from which we sample $n = 700, 5000$ samples with $d = 200$, $p \in {0.01, ..., 0.4}$ and $\lambda = {0, ..., 10}$
We plot the results of this experiment in \autoref{anisotropic_experiment} in which we plot the average value of $|\inner{\uhat, \ustar}|$ over 10 runs of the gradient-based projection pursuit algorithm.
Here, we can observe that with smaller values for $p$ and larger values for $\lambda$, the problem becomes easier to recover the signal direction.

\begin{figure}[H]
    \centering
        \begin{subfigure}[b]{0.4\textwidth}
            \centering
            \includegraphics[width=\textwidth]{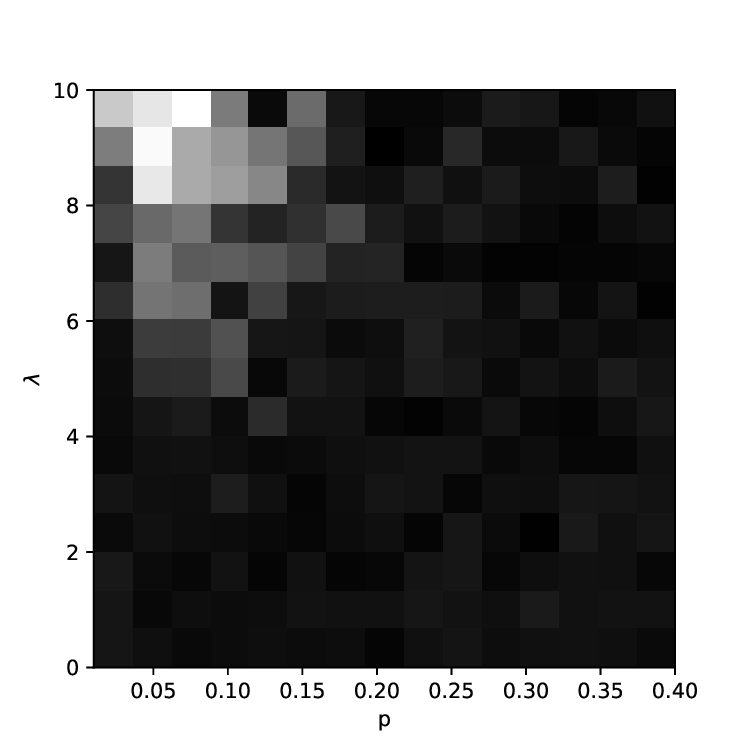}
            \caption{$n = 700$}
        \end{subfigure}
        \begin{subfigure}[b]{0.4\textwidth}
            \centering
            \includegraphics[width=\textwidth]{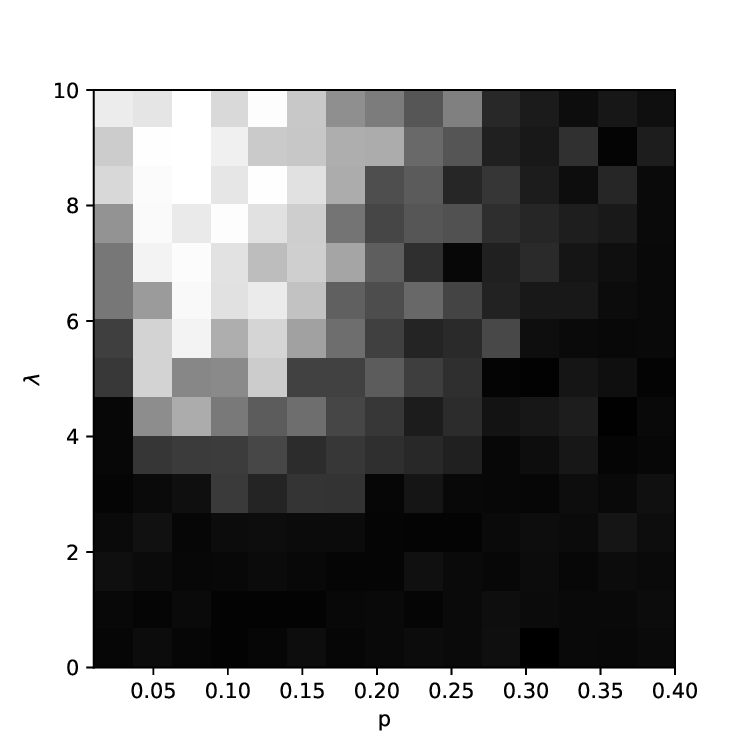}
            \caption{$n = 5000$}
        \end{subfigure}
    \caption{
        Plots of $|\inner{\uhat, \ustar}|$ where $\uhat$ was generated by the gradient-based algorithm with the \textbf{ReLU2} projection index. 
        White pixels denote values close to 1, and black pixels denote values close to 0.
    }
    \label{anisotropic_experiment}
\end{figure}

\subsection{Planted Vector in other Distributions than Gaussians}
The planted vector setting assumes that the orthogonal space to the signal direction is Gaussian.
Here, we will consider alternative distributions ($\inner{\bX, \ustar}$) for the signal and orthogonal directions.
Thus we will choose a random orthonormal basis $(\ustar, \be_2, ..., \be_d)$
For $\inner{X, \ustar}$ we will choose these distributions:
\begin{center}
\begin{tabular}{ c|c } 
\textbf{Abbreviation} & \textbf{Distribution} \\ 
\hline
GM & $0.9 \Normal(0, 1) + 0.1 \Normal(5, 1)$ \\ 
\hline
IC & Imbalanced Clusters with $p = 0.1$ \\ 
\hline
BR & Bernoulli Rademacher with $p = 0.1$ \\ 
\end{tabular}
\end{center}

For $\inner{X, \be_i}$ with $i \in \{ 2, ..., d\}$ we will choose the following distributions.

\begin{center}
\begin{tabular}{ c|c } 
\textbf{Abbreviation} & \textbf{Distribution} \\ 
\hline
Normal & $\Normal(0, 1)$ \\ 
\hline
Rademacher & Rademacher distribution\\ 
\hline
Heavy-Tailed & Student-t distribution with $\nu = 2$ degrees of freedom\\ 
\hline
Skewed & Skew normal distribution with $\alpha = 3$
\end{tabular}
\end{center}

The results of the experiment are plotted in \autoref{other_ambient} where we run the experiment for dimension $d = 100$ with different numbers of samples.
The data is whitened before applying the projection pursuit algorithm.
We test 2 different methods of choosing the best projection. 
First, we choose the algorithm as proposed in Algorithm 1.
Secondly, we just run gradient ascent once and choose the direction maximizing $\myabs{\inner{\uhat, \ustar}}$.
This is not implementable in practice but demonstrates that the gradient ascent algorithm still converges to the signal direction with a heavy-tailed distribution as the distribution orthogonal to the signal.
We believe this could be fixed using a more robust function for $\psi(\cdot)$.

We can observe that using Rademacher distributions instead of a Gaussian does not seem to change how easily the signal direction can be recovered.
It can also be observed that the heavy-tailed Gaussian Mixture makes recovery significantly more difficult.
We especially observe that our choice of projection index to detect converged projections fails for heavy-tailed distributions.

\begin{figure}
    \centering
    \begin{subfigure}[b]{0.49\textwidth}
        \includegraphics[width=\linewidth]{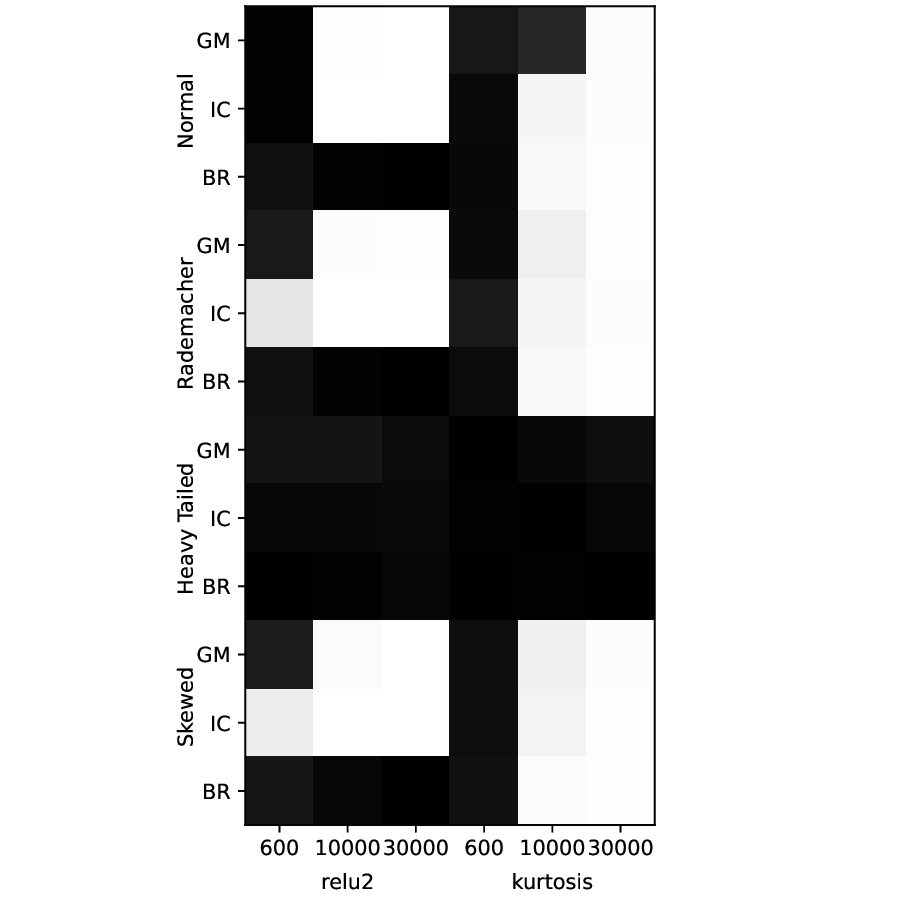}
        \caption{Proposed algorithm}
    \end{subfigure}
    \begin{subfigure}[b]{0.49\textwidth}
        \includegraphics[width=\linewidth]{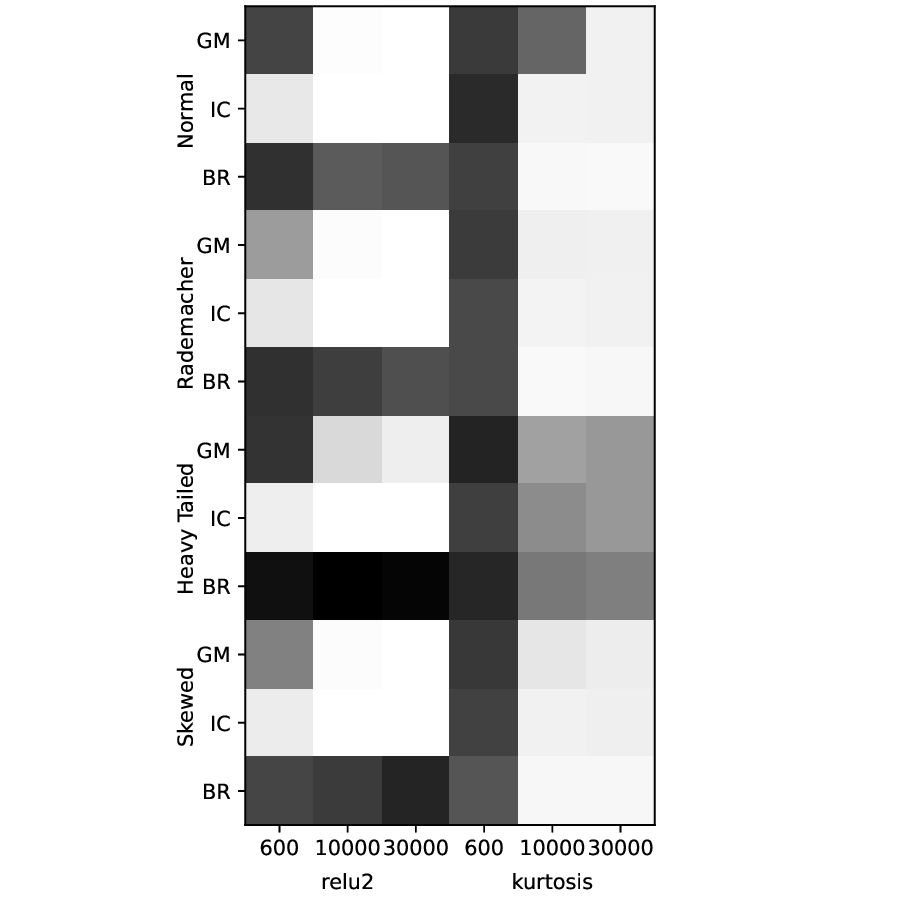}
        \caption{Best direction found}
    \end{subfigure}
    
    \caption{
        On the x-axis: number of samples $n = 600, 10000, 30000$ and projection index \{relu2, kurtosis\}.\\
        On the y-axis: distribution in the signal direction \{ Gaussian Mixture(GM), Imbalanced Clusters(IC), Bernoulli Rademacher(BR) \} and the distribution of the orthogonal space \{ Normal, Rademacher, Heavy Tailed, Skewed\}.\\
        We plot the inner product $\inner{\uhat, \ustar}$ of the recovered direction $\uhat$.
        White pixels denote values close to 1, and black pixels denote values close to 0.
    }
    \label{other_ambient}
\end{figure}

\end{document}